\pdfoutput=1
\documentclass[11pt]{article}
\usepackage{fullpage}
\usepackage{mathtools,amsfonts,amsthm,amssymb,bbm}
\usepackage{xcolor}
\usepackage[backref=section,colorlinks,citecolor=blue,linkcolor=magenta,bookmarks=true]{hyperref}
\usepackage[nameinlink]{cleveref}
\usepackage{authblk}
\usepackage[algo2e,ruled]{algorithm2e}

\newcommand{\1}[1]{\mathbbm{1}\left[#1\right]}
\newcommand{\argmax}{\operatorname*{argmax}}
\newcommand{\D}{\mathcal{D}}

\newcommand{\dTV}{d_{\textrm{TV}}}
\newcommand{\eps}{\epsilon}
\newcommand{\Ex}[2]{\operatorname*{\mathbb{E}}_{#1}\left[#2\right]}
\newcommand{\Exp}{\mathrm{Exp}}
\newcommand{\F}{\mathcal{F}}

\newcommand{\muP}{\mu^{(P)}}
\newcommand{\muQ}{\mu^{(Q)}}
\newcommand{\mux}{\mu^{(x)}}
\newcommand{\N}{\mathcal{N}}
\newcommand{\norm}[1]{\|#1\|}
\newcommand{\poly}{\mathrm{poly}}
\newcommand{\pr}[2]{\Pr_{#1}\left[#2\right]}
\newcommand{\R}{\mathbb{R}}
\renewcommand{\Re}{\operatorname{Re}}
\renewcommand{\S}{\mathbb{S}}

\newcommand{\Unif}{\mathrm{Uniform}}
\newcommand{\Vol}{\mathrm{Vol}}

\theoremstyle{plain}
\newtheorem{theorem}{Theorem}[section]

\newtheorem{lemma}[theorem]{Lemma}
\newtheorem{corollary}[theorem]{Corollary}
\theoremstyle{definition}

\theoremstyle{remark}
\newtheorem{remark}[theorem]{Remark}

\newtheorem{claim}[theorem]{Claim}

\title{A Fourier Approach to Mixture Learning}

\date{}

\author{Mingda Qiao\thanks{Part of this work was done while working as an intern at Google Research.}}
\affil[1]{Stanford University\\\texttt{mqiao@stanford.edu}}
\author[2]{Guru Guruganesh}
\author[2]{Ankit Singh Rawat}
\author[2]{Avinava Dubey}
\author[3]{Manzil~Zaheer}
\affil[2]{Google Research\\\texttt{\{gurug,ankitsrawat,avinavadubey\}@google.com}}
\affil[3]{Google DeepMind\\\texttt{manzilzaheer@google.com}}

\begin{document}

\maketitle

\begin{abstract}
We revisit the problem of learning mixtures of spherical Gaussians. Given samples from mixture $\frac{1}{k}\sum_{j=1}^{k}\N(\mu_j, I_d)$, the goal is to estimate the means $\mu_1, \mu_2, \ldots, \mu_k \in \R^d$ up to a small error. The hardness of this learning problem can be measured by the \emph{separation} $\Delta$ defined as the minimum distance between all pairs of means. Regev and Vijayaraghavan~\cite{RV17} showed that with $\Delta = \Omega(\sqrt{\log k})$ separation, the means can be learned using $\poly(k, d)$ samples, whereas super-polynomially many samples are required if $\Delta = o(\sqrt{\log k})$ and $d = \Omega(\log k)$. This leaves open the low-dimensional regime where $d = o(\log k)$.
    
In this work, we give an algorithm that efficiently learns the means in $d = O(\log k/\log\log k)$ dimensions under separation $d/\sqrt{\log k}$ (modulo doubly logarithmic factors). This separation is strictly smaller than $\sqrt{\log k}$, and is also shown to be necessary. Along with the results of~\cite{RV17}, our work almost pins down the critical separation threshold at which efficient parameter learning becomes possible for spherical Gaussian mixtures. More generally, our algorithm runs in time $\poly(k)\cdot f(d, \Delta, \eps)$, and is thus fixed-parameter tractable in parameters $d$, $\Delta$ and $\eps$.

Our approach is based on estimating the Fourier transform of the mixture at carefully chosen frequencies, and both the algorithm and its analysis are simple and elementary. Our positive results can be easily extended to learning mixtures of non-Gaussian distributions, under a mild condition on the Fourier spectrum of the distribution.
\end{abstract}

\section{Introduction}
Gaussian mixture models (GMMs) are one of the most well studied models for a population with different components. 
A GMM defines a distribution over the $d$-dimensional Euclidean space as the weighted sum of normal distributions $\sum_{i=1}^k w_i \cdot \N(\mu_i, \Sigma_i)$, which are specified by following quantities: the number of components $k \in \mathbb{N}$, the component means $\mu_i \in \R^d$, the component covariances $\Sigma_i \in \R^{d \times d}$, which are positive definite matrices, and the weights $w_i \geq 0$ that sum up to $1$.
In this work, we consider the uniform spherical case, where the weights $w_i$ are uniform  $w_i = \frac{1}{k}$ and the covariance matrix $\Sigma_i = I_d$ is the identity matrix.
The central problem in this setup is to \emph{efficiently} estimate the means $\mu_1,\dots,\mu_k$.
To avoid degenerate cases such as when some of the means are the same, it is common to parameterize the  problem  by the separation of the means $\Delta$, which guarantees that $\norm{\mu_i - \mu_j}_2 \geq \Delta$ for all $i\neq j$. 

More precisely, the problem is to estimate the means $\mu_1,\dots,\mu_k \in \R^d$ up to an error $\eps$ with runtime that is  $\poly(k,\frac{1}{\eps}, d)$ with as small a separation $\Delta$ as possible.
There has been a long line of work on this problem which we survey in~\Cref{sec:related}.

Recently, Regev and Vijayaraghavan~\cite{RV17} showed that a separation $\Delta =\Omega(\sqrt{\log k})$ is strictly necessary when the dimension $d=\Omega(\log k)$. 
Two natural questions arise immediately. First, if $\Delta = \sqrt{\log k}$ is sufficient when $d=\Omega(\log k)$. Although the original work of \cite{RV17} showed that it was information theoretically possible, an actual efficient algorithm was only recently developed by \cite{LL22} (who show nearly tight results). The second main question is determining  the optimal separation in low dimensions when $d= o(\log k)$. Previously, even in $O(1)$ dimensions, the exact separation necessary was unknown. In this paper, we settle the second question and give nearly optimal bounds on the separation necessary in low dimensions (see~\Cref{fig:landscape} for more details).

\subsection{Overview of Results}
We begin with a few definitions. A point set $\{x_1, x_2, \ldots, x_k\}$ is called \emph{$\Delta$-separated} if $\|x_j - x_{j'}\|_2 \ge \Delta$ for any $j \ne j'$. We say that a Gaussian mixture is $\Delta$-separated (or has separation $\Delta$) if the means of its components are $\Delta$-separated. Two point sets $\{u_1, \ldots, u_k\}$ and $\{v_1, \ldots, v_k\}$ are \emph{$\eps$-close} if for some permutation $\sigma$ over $[k]$, $\|u_j - v_{\sigma(j)}\|_2 \le \eps$ holds for every $j$.

Our main result is an algorithm that efficiently learns the parameters of a mixture of $k$ spherical Gaussians under separation $\Delta \approx \frac{d}{\sqrt{\log k}} \cdot \poly(\log\log k)$ in $d = O(\log k /\log\log k)$ dimensions. In the low-dimensional regime, this separation is strictly smaller than $\min\{\sqrt{\log k}, \sqrt{d}\}$, the smallest separation under which previous algorithms could provably learn the parameters in $\poly(k)$ time.

\begin{theorem}[Upper bound, informal]\label{thm:upper-gaussian-informal}
    Let $P$ be a uniform mixture of $k$ spherical Gaussians in $d = O\left(\frac{\log k}{\log\log k}\right)$ dimensions with separation $\Delta = \Omega\left(\frac{d\sqrt{\log((\log k)/ d)}}{\sqrt{\log k}}\right)$. There is a $\poly(k)$-time algorithm that, given samples from $P$, outputs $k$ vectors that are w.h.p.\ $\eps$-close to the true means for $\eps = O(\Delta/\sqrt{d})$.
\end{theorem}

See~\Cref{thm:upper-gaussian} and Remark~\ref{remark:upper-gaussian} for a more formal statement of our algorithmic result, which holds for a wider range of separation $\Delta$ and accuracy parameter $\eps$. Our learning algorithm is provably correct for arbitrarily small $\Delta, \eps > 0$ (possibly with a longer runtime), whereas for most of the previous algorithms, there is a breakdown point $\Delta^*$ such that the algorithm is not known to work when $\Delta < \Delta^*$. Two exceptions are the algorithms of~\cite{MV10,BS10}, both of which allow an arbitrarily small separation but run in $e^{\Omega(k)}$ time. We also remark that the runtime of our algorithm scales as $\tilde O(k^3)\cdot f(d, \Delta, \eps)$, and is thus fixed-parameter tractable in parameters $d$, $\Delta$ and $\eps$.\footnote{While we focus on the uniform-weight case for simplicity, \Cref{thm:upper-gaussian-informal} can be easily extended to the setting where each weight is in $[1/(Ck), C/k]$ for some constant $C > 1$.}

We complement \Cref{thm:upper-gaussian-informal} with an almost-matching lower bound, showing that the $d/\sqrt{\log k}$ separation is necessary for efficient parameter learning in low dimensions.

\begin{theorem}[Lower bound, informal]\label{thm:lower-gaussian-informal}
    For $d = O\left(\frac{\log k}{\log\log k}\right)$ and $\Delta = o\left(\frac{d}{\sqrt{\log k}}\right)$, there are two mixtures of $k$ spherical Gaussians in $\R^d$ such that: (1) both have separation $\Delta$; (2) their means are not $(\Delta/2)$-close; and (3) the total variation (TV) distance between them is $k^{-\omega(1)}$.
\end{theorem}

See~\Cref{thm:lower} for a more formal version of the lower bound. 

\Cref{thm:upper-gaussian-informal}~and~\Cref{thm:lower-gaussian-informal} together nearly settle the polynomial learnability of spherical Gaussian mixtures in the low-dimensional regime. Up to a doubly logarithmic factor, the ``critical separation'' where efficient learning becomes possible is $d/\sqrt{\log k}$. To the best of our knowledge, this was previously unknown even for $d = O(1)$.\footnote{An exception is the $d = 1$ case: A result of~\cite{Moi15} implies that $\Delta = \Omega(1/\sqrt{\log k})$ suffices (see Section~\ref{sec:related}), and a matching lower bound was given by~\cite{MV10}.}

See Figure~\ref{fig:landscape} below for a plot of our results in the context of prior work. The green regions cover the parameters $(\Delta, d)$ such that mixtures of $k$ spherical Gaussians in $d$ dimensions with separation $\Delta$ are learnable (up to $O(\Delta)$ error) using $\poly(k)$ samples.\footnote{In terms of the computational complexity, all the green regions (except a small portion of Region~III) admit efficient algorithms. The algorithm of~\cite{LL22} is  efficient when $\Delta = \Omega((\log k)^{1/2+c})$ for any $c > 0$ and thus almost covers Region~III. For Region~IV, \cite{RV17} gave an efficient algorithm only for $d = O(1)$, whereas \Cref{thm:upper-gaussian-informal} covers the entire Region~IV.} The red regions contain the parameters under which polynomial-sample
learning is provably impossible. 

\begin{figure}[t]
    \centering
    \includegraphics[scale=0.35]{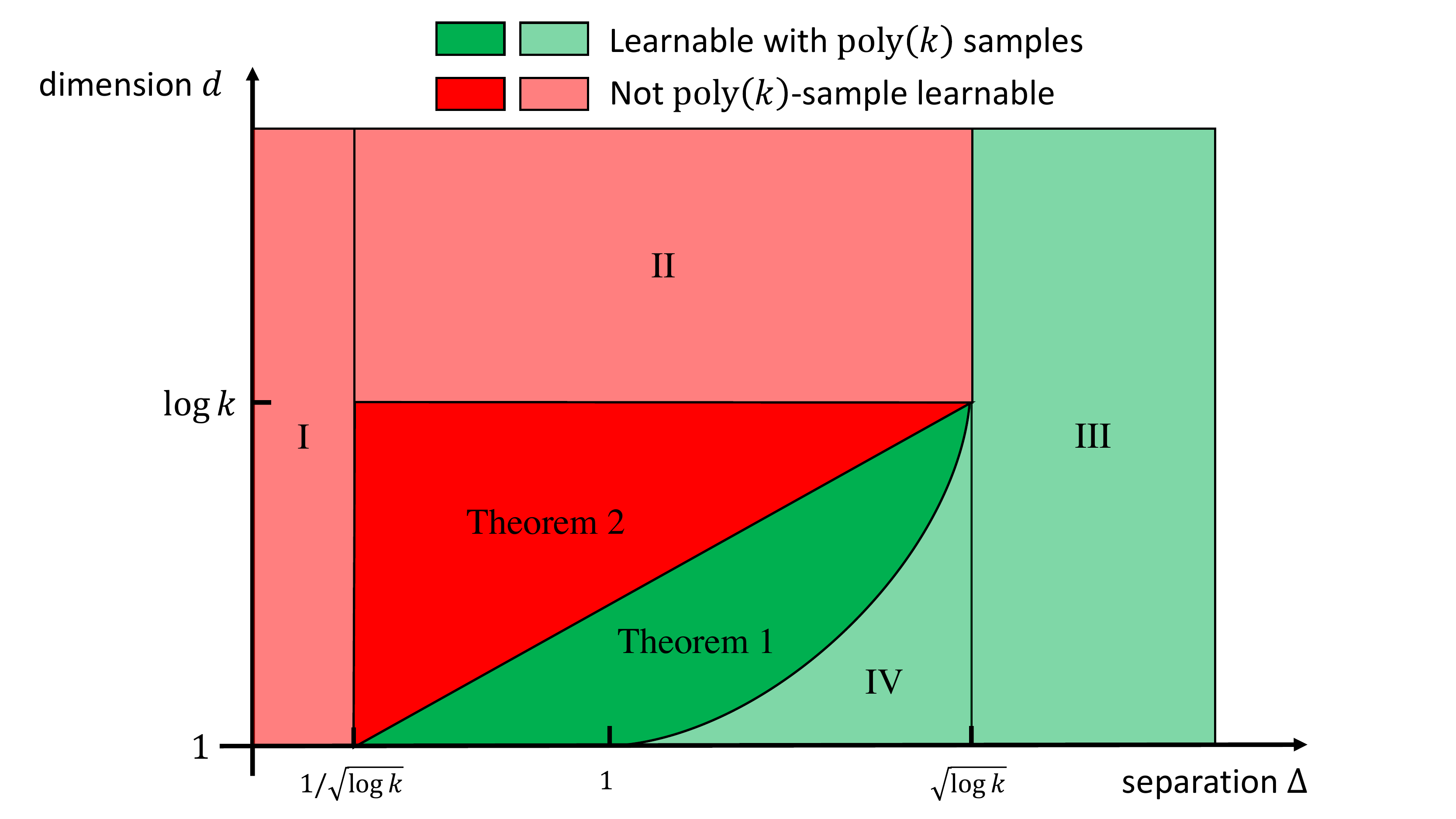}
    \caption{
    Region I is a direct corollary of~\cite[Proposition 15]{MV10}. Regions II, III, and IV are shown by~\cite[Theorems 1.2, 1.3, and 1.4]{RV17} respectively. The upper boundary of Region IV is the curve $\Delta = \sqrt{d}$. Theorems \ref{thm:upper-gaussian-informal}~and~\ref{thm:lower-gaussian-informal} settle the learnability in the remaining area, by proving that the line $\Delta = d/\sqrt{\log k}$ is the boundary between polynomial and super-polynomial sample complexities (up to a doubly-logarithmic factor).}
    \label{fig:landscape}
    
\end{figure}

The algorithm that underlies \Cref{thm:upper-gaussian-informal} can be easily extended beyond the spherical Gaussian case. The following more general result states that for any distribution $\D$ whose the Fourier transform does not decay too fast, we can efficiently learn the parameters of a mixture of $k$ translated copies of $\D$. In the following, let $\D_{\mu}$ denote the distribution of $X + \mu$ when $X$ is drawn from $\D$.

\begin{theorem}[Learning more general mixtures, informal]\label{thm:upper-general-informal}
    Let $P = \frac{1}{k}\sum_{j=1}^{k}\D_{\mu_j}$ for $\Delta$-separated $\mu_1, \ldots, \mu_k \in \R^d$. There is an algorithm that, given $\eps > 0$ and samples from $P$, runs in time
    \[
        \poly\left(k, \Delta/\eps, 1/\delta, \max_{\|\xi\|_2 \le M}\left|\Ex{X \sim \D}{e^{i\xi^{\top}X}}\right|^{-1}\right)
    \]
    for some $\delta = \delta(\D, \eps)$ and $M = M(k, d, \Delta, \eps)$, and outputs $\hat\mu_1, \ldots, \hat\mu_k$ that are w.h.p.\ $\eps$-close to the true parameters.
\end{theorem}

See \Cref{thm:upper-general-testing} and \Cref{corollary:upper-general} for a more formal statement of the runtime. \Cref{thm:upper-general-informal} applies to many well-known distribution families that are specified by either a ``location parameter'' or a ``scale parameter''. \Cref{table:general} gives a few examples of applying \Cref{thm:upper-general-informal} to mixtures of single-parameter univariate distributions; see  \Cref{sec:examples}~and~\Cref{sec:exponential} for more details. 

\begin{table}[h]
    \renewcommand{\arraystretch}{1.7}
    \centering
    \begin{tabular}{|c|c|c|c|}
    \hline
    Distribution & Parameter & Density Function  &  Runtime \\ \hline
     Cauchy & $\mu$ & $\frac{1}{\pi(1+(x - \mu)^2)}$ & $O(k^3)\cdot e^{O(\sqrt{\log k}/\Delta)}$ \\ \hline 
     Logistic & $\mu$ & $\frac{e^{-(x - \mu)}}{(1 + e^{-(x - \mu)})^2}$ & $ O(k^3)\cdot e^{O(\sqrt{\log k}/\Delta)}$ \\ \hline
     Laplace & $\mu$ & $\frac{1}{2}e^{-|x - \mu|}$ & $\tilde O(k^3/\Delta^5)$ \\ \hline
     Exponential & $\ln\lambda$ & $\lambda e^{-\lambda x}\cdot \1{x \ge 0}$ & $O(k^3)\cdot e^{O(\sqrt{\log k}/\Delta)}$ \\ \hline
    \end{tabular}
    \caption{Implication of \Cref{thm:upper-general-informal} for learning various families of mixtures of univariate distributions beyond Gaussians, assuming that the parameters of different components are $\Delta$-separated. The algorithm outputs $k$ parameters that are $O(\Delta)$-close to the true parameters.} 
    \label{table:general}
\end{table}

\paragraph{Limitations of our work.} The main limitation is that the positive results only apply to the regime that the dimension $d$ is logarithmic in the number of clusters, and that all clusters are translated copies of the same distribution. A concrete future direction would be to extend our results to learning mixtures of general Gaussians, even in one dimension.

\subsection{Proof Overview}\label{sec:proof-overview}

For simplicity, we focus on the one dimensional case that $P = \frac{1}{k}\sum_{j=1}^k\N(\mu_j, 1)$ for $\Delta$-separated means $\mu_1, \ldots, \mu_k \in \R$. We index the components such that $|\mu_1| \le |\mu_2| \le \cdots \le |\mu_k|$, and focus on the following testing problem: Given $\eps > 0$ and samples from $P$, determine whether $\mu_1 = 0$ or $\mu_1 \ge \eps$, assuming that one of them is true. Note that this testing problem is not harder than estimating $\mu_1, \ldots, \mu_k$ up to error $\epsilon / 3$---in the former case that $\mu_1 = 0$, one of the mean estimates would fall into $[-\epsilon/3, \epsilon/3]$, whereas all of them must be outside $(-2\epsilon/3, 2\epsilon/3)$ in the latter case. Conversely, as we will prove in Section~\ref{sec:learning-to-testing}, an algorithm for the testing version can be used for recovering the means as well.

\paragraph{Examine the Fourier spectrum.} We start by examining the Fourier transform of $P$, $(\F P)(\xi)
\coloneqq \Ex{X \sim P}{e^{i\xi X}}$, more commonly known in the literature as the \emph{characteristic function}. 
Since the Fourier transform of a Gaussian is still a Gaussian, and a translation in the time domain shifts the phase in the frequency domain, we have
{\small
\begin{equation}\label{eq:Amu-intro}
    \Ex{X \sim P}{e^{i\xi X}}
=   \frac{1}{k}\sum_{j=1}^{k}\Ex{X \sim \N(\mu_j, 1)}{e^{i\xi X}}
=   \frac{e^{-\xi^2/2}}{k}\sum_{j=1}^{k}e^{i\mu_j\xi}.
\end{equation}
}
We will focus on the quantity $A_{\mu}(\xi) \coloneqq \sum_{j=1}^{k}e^{i\mu_j\xi}$, which is the ``total phase'' over the $k$ components of $P$. \Cref{eq:Amu-intro} essentially states that $A_{\mu}(\xi)$ can be estimated by averaging $e^{i\xi X}$ over samples $X \sim P$.

The key observation is that each term $e^{i\mu_j\xi}$ of $A_{\mu}(\xi)$ behaves quite differently depending on the magnitude of $\mu_j$: If $\mu_j = 0$, $e^{i\mu_j\xi} = 1$ is a constant, whereas $e^{i\mu_j\xi}$ is a high-frequency wave when $|\mu_j|$ is large. This suggests that the two cases ($\mu_1 = 0$ and $|\mu_1| \ge \eps$) can be distinguished by estimating $A_{\mu}(\xi)$ at different frequencies. The cost of estimating $A_{\mu}(\xi)$, however, depends heavily on the frequency $\xi$ -- \Cref{eq:Amu-intro} together with a simple concentration bound shows that $O(k^2)\cdot e^{O(\xi^2)}$ samples are sufficient for estimating $A_{\mu}(\xi)$ up to a constant error. 

Therefore, the crux of this approach is to find the minimum $M > 0$ such that the two cases can be distinguished by estimating $A_{\mu}(\xi)$ over $\xi \in [-M, M]$. (This is known as the \emph{super-resolution} problem, which we discuss in \Cref{sec:related}.) The sample complexity of the testing problem can then be roughly bounded by $O(k^2)\cdot e^{O(M^2)}$. In the following, we explore different ways of picking $\xi$ from the range $[-M, M]$.

\paragraph{Choosing $\xi$ randomly.} Our overall strategy is to draw $\xi$ randomly from some distribution $\D_{\xi}$ over interval $[-M, M]$ and evaluate $\Ex{\xi\sim\D_{\xi}}{A_{\mu}(\xi)}$. We will argue that this expectation is very close to its first term $\Ex{\xi\sim\D_{\xi}}{e^{i\mu_1\xi}}$, which takes different values depending on whether $\mu_1 = 0$ or $|\mu_1| \ge \eps$. Then, $\D_{\xi}$ needs to be chosen such that: (1) There is a gap in the value of $\Ex{\xi\sim\D_{\xi}}{e^{i\mu_1\xi}}$ between the two cases; (2) $\left|\sum_{j=2}^{k}\Ex{\xi}{e^{i\mu_j\xi}}\right|$ is small enough for the gap in the first term to be easily identified.

As a warmup, let $\D_{\xi}$ be the uniform distribution over $[0, M]$. A simple calculation shows that the gap 
in the first term  $\Ex{\xi\sim\D_{\xi}}{e^{i\mu_1\xi}}$ between the two cases ($\mu_1 = 0$ or $|\mu_1| \ge \eps$) is lower bounded by $\Omega(\min\{M\eps, 1\})$,
whereas the contribution from the $j$-th component satisfies $\left|\Ex{\xi\sim\D_{\xi}}{e^{i\mu_j\xi}}\right| = O\left(\frac{1}{M|\mu_j|}\right)$. Furthermore, the $\Delta$-separation between the means implies $|\mu_j| = \Omega(\Delta j)$. Thus,
{\small
\[
    \left|\sum_{j=2}^{k}\Ex{\xi}{e^{i\mu_j\xi}}\right|
\lesssim \sum_{j=2}^{k}\frac{1}{M|\mu_j|}
\lesssim \frac{1}{M\Delta}\sum_{j=2}^{k}\frac{1}{j}
\lesssim \frac{\log k}{M\Delta}.
\]
}
The above is much smaller than $\min\{M\eps, 1\}$ if we set
$M \gtrsim \max\left\{\frac{\log k}{\Delta}, \sqrt{\frac{\log k}{\Delta\eps}}\right\}$. Unfortunately, even when $\Delta$ and $\eps$ are constants, we have $O(k^2)\cdot e^{O(M^2)} = k^{O(\log k)}$ and the resulting sample complexity is already super-polynomial in $k$.

\paragraph{A better choice of $\D_{\xi}$.} It turns out that choosing $\D_{\xi}$ to be a truncated Gaussian leads to a much lower sample complexity. For some $\sigma \ll M$, we draw $\xi \sim \N(0, \sigma^2)$ and then truncate it to $[-M, M]$. Without the truncation, the expectation of $e^{i\mu_j\xi}$ has a nice closed form:
\[
    \Ex{\xi \sim \N(0, \sigma^2)}{e^{i\mu_j\xi}}
=   e^{-\sigma^2\mu_j^2/2},
\]
which is exactly the Fourier weight of $\N(0, \sigma^2)$ at frequency $\mu_j$. Note that this decreases very fast as $|\mu_j|$ grows, compared to the previous rate of $\frac{1}{M|\mu_j|}$ when $\D_{\xi}$ is uniform.

It again follows from a simple calculation that: (1) Depending on whether $\mu_1 = 0$ or $|\mu_1| \ge \eps$, the gap between $\Ex{\xi}{e^{i\mu_1\xi}}$ is $\Omega(\min\{\sigma^2\eps^2, 1\})$; (2) The total contribution from $j = 2, 3, \ldots, k$ is upper bounded by
{\small
\[
    \left|\sum_{j=2}^{k}\Ex{\xi\sim\N(0, \sigma^2)}{e^{i\mu_j\xi}}\right|
=   \sum_{j=2}^{k}e^{-\sigma^2\mu_j^2/2}
\le    \sum_{j=2}^{k}e^{-\Omega(\sigma^2\Delta^2j^2)}
=   e^{-\Omega(\sigma^2\Delta^2)},
\]
}
where the second step applies $|\mu_j| = \Omega(\Delta j)$ and the last step holds assuming $\sigma = \Omega(1/\Delta)$. In addition, we need to deal with the error incurred by the truncation. The Gaussian tail bounds imply that $|\xi| \ge M$ happens with probability $e^{-\Omega(M^2/\sigma^2)}$. Since $\left|A_{\mu}(\xi)\right| \le k$ for any $\xi$, the noise from the truncation is at most $k\cdot e^{-\Omega(M^2/\sigma^2)}$ in magnitude.

It remains to choose $M, \sigma > 0$ that satisfy the two inequalities:
\[
    e^{-\Omega(\sigma^2\Delta^2)} \ll \min\{\sigma^2\eps^2, 1\}
\quad\text{and}\quad
    k\cdot e^{-\Omega(M^2/\sigma^2)} \ll \min\{\sigma^2\eps^2, 1\}.
\]
It suffices to choose $\sigma \lesssim \frac{1}{\Delta}\sqrt{\log\frac{\Delta}{\eps}}$ and $M \lesssim \frac{1}{\Delta}\sqrt{\log\frac{\Delta}{\eps}}\sqrt{\log\frac{k\Delta}{\eps}}$. For $\eps = \Omega(\Delta)$, the sample complexity $O(k^2)\cdot e^{O(M^2)}$ reduces to $k^{O(1/\Delta^2)}$, which is polynomial in $k$ for any fixed $\Delta$.

We note that in the above derivation, the assumption that each cluster of $P$ is a Gaussian is only applied in \Cref{eq:Amu-intro} (through the Fourier transform). For any distribution $\D$ over $\R$ and $P = \frac{1}{k}\sum_{j=1}^{k}\D_{\mu_j}$, the quantity $A_{\mu}(\xi) = \sum_{j=1}^{k}e^{i\mu_j\xi}$ can still be read off from the Fourier transform of $P$ at frequency $\xi$, except that the extra factor $e^{-\xi^2/2}$ becomes $(\F\D)(\xi)$, the Fourier transform of $\D$ at frequency $\xi$. This observation leads to our algorithm for learning a mixture of multiple translated copies of $\D$ (\Cref{thm:upper-general-informal}).

\paragraph{Gaussian truncation.} For the spherical Gaussian case, however, the above only gives an algorithm with runtime $k^{O(1/\Delta^2)}$, which becomes super-polynomial as soon as $\Delta = o(1)$, falling short of achieving the near-optimal separation of $\sqrt{\frac{\log\log k}{\log k}}$ in \Cref{thm:upper-gaussian-informal}.

We further improve our algorithm for the spherical Gaussian case using a ``Gaussian truncation'' technique. Intuitively, when deciding whether the mixture $P$ contains a cluster with mean zero, a sample $X \sim P$ is much more informative when $|X|$ is small. This motivates us to focus on samples with a small magnitude via a truncation.

We apply such a truncation in a soft way---weighting each sample $X$ with $e^{-X^2/2}$. This turns out to be sufficiently effective while keeping the entire analysis simple. Note that the weighting effectively multiplies the mixture $P$ with the standard Gaussian $\N(0, 1)$ pointwise. The result is still a (un-normalized) mixture of Gaussians: Up to a constant factor, the pointwise product is
$
    \frac{1}{k}\sum_{j=1}^{k}e^{-\mu_j^2/4}\cdot\N(\mu_j / 2, 1/2).
$
Consequently, if we repeat the analysis from previous paragraphs to this weighted mixture, the noise coming from components with large $|\mu_j|$ become even smaller. This eventually allows us to learn the mixture efficiently at separation $\Delta = \sqrt{\frac{\log\log k}{\log k}}$. We prove \Cref{thm:upper-gaussian-informal} via a natural extension of this analysis to the $d$-dimensional case.

\paragraph{Lower bound.} Our proof of \Cref{thm:lower-gaussian-informal} follows an approach similar to those in the previous lower bound proofs of~\cite{MV10,HP15,RV17}: First, construct two sets of well-separated points $\{\muP_1, \ldots, \muP_k\}$ and $\{\muQ_1, \ldots, \muQ_k\}$ with (approximately) matching lower-order moments, i.e., $\frac{1}{k}\sum_{j=1}^{k}\left[\muP_j\right]^{\otimes t} \approx \frac{1}{k}\sum_{j=1}^{k}\left[\muQ_j\right]^{\otimes t}$ for every small $t$. Then, show that these matching moments imply that after a convolution with Gaussian, the resulting mixtures are close in TV-distance and thus hard to distinguish.

In more detail, similar to the proof of~\cite{RV17}, we start by choosing $N$ arbitrary points from a small $\ell_2$ ball, such that they are $\Delta$-separated. The main difference is that~\cite{RV17} pick $N \gg k$, and show that among all the $\binom{N}{k}$ possible mixtures, there are at least two mixtures with similar moments via a pigeon-hole type argument. In contrast, we work in the $N \ll k$ regime, and obtain two point sets with matching lower-order moments by slightly perturbing these $N$ points in opposite directions. The existence of such a good perturbation is shown via a careful application of the Borsuk--Ulam Theorem, which is inspired by a similar application in~\cite{HP15}.

\subsection{Related Work} \label{sec:related}

\paragraph{Learning Gaussian mixture models.}
Most closely related to this paper is the line of work in the theoretical computer science literature on algorithms that provably learn mixtures of Gaussians. The pioneering work of \cite{Das99} showed that an $\Omega(\sqrt{d})$ separation between the means of different components is sufficient for the samples to be easily clustered. Several subsequent work (e.g., \cite{VW04,AK05,AM05,KSV05,DS07,BV08}) further generalized this result and the separation condition is relaxed to $\Delta = \Omega((\min\{k, d\})^{1/4}\cdot\poly(\log k, \log d))$ by~\cite{VW04}. We refer the readers to~\cite{RV17} for a more detailed survey of these results.

\cite{MV10}~and~\cite{BS10} gave the first algorithms for learning general mixtures of Gaussians (with different unknown covariances) under an arbitrarily small separation between different components. Both algorithms are based on the method of moments, and run in time polynomial in $d$ and the inverse of the minimum TV-distance between different components when $k = O(1)$. For large $k$, however, the runtime is exponential in $k$.

\cite{RV17} showed that $\poly(k, d)$ samples are sufficient to estimate the means of a mixture of spherical Gaussians, if the means are separated by $\Delta = \Omega(\sqrt{\log k})$. Unfortunately, their algorithm involves an exhaustive search and is computationally inefficient. Subsequently, three concurrent work~\cite{DKS18,HL18,KSS18} developed algorithms based on the Sum-of-Squares hierarchy that run in time $(dk)^{\poly(1/\gamma)}$ and learn mixtures with separation $\Delta = \Omega(k^{\gamma})$. In particular, setting $\gamma \approx \frac{\log\log k}{\log k}$ gives a quasi-polynomial time algorithm that achieves the $\sqrt{\log k}$ separation in~\cite{RV17}. A very recent work of~\cite{LL22} made further progress towards learning $\Omega(\sqrt{\log k})$-separated mixtures efficiently, by giving a polynomial-time algorithm that succeeds under separation $(\log k)^{1/2+c}$ for any constant $c > 0$.

\paragraph{Lower bounds.} On the lower bound side, \cite{MV10} first proved, by explicitly constructing a hard instance, that learning a mixture of $k$ Gaussians with separation $\Delta = \Omega(1/\sqrt{k})$ requires $e^{\Omega(k)}$ samples even in $d = 1$ dimension. \cite{HP15} focused on the regime where $k = O(1)$ and the recovery error $\eps$ goes to zero, and showed that $\Omega(\eps^{2-6k})$ samples are necessary. \cite{ABG+14} proved a lower bound that strengthens the result of~\cite{MV10}: Separation $\Delta = 1/\poly(k)$ is insufficient for the mixture to be learnable with polynomially many samples, even when the means of the Gaussians are drawn uniformly at random from $[0,1]^d$ for $d = O(\log k/\log\log k)$.

\cite{RV17} proved that in $d = \Omega(\log k)$ dimensions, no polynomial-sample algorithm exists if $\Delta = o(\sqrt{\log k})$. In particular, even when the means are chosen randomly, the resulting mixture is still hard to learn with high probability. This lower bound, together with their positive result for the $\Delta = \Omega(\sqrt{\log k})$ regime, shows that $\sqrt{\log k}$ is the critical separation threshold in high dimensions. 

If we restrict our attention to statistical query (SQ) algorithms, \cite{DKS17} proved that any SQ algorithm needs $d^{\Omega(k)}$ queries if $d \ge \poly(k)$ and non-spherical components are allowed.

\paragraph{Density estimation.} In contrast to the parameter estimation setting that we focus on, the \emph{density estimation} setting requires the learning algorithm to output (the representation of) a hypothesis distribution $\hat P$ that is close to the mixture $P$ in TV-distance. \cite{ABH+20} proved that the sample complexity of learning Gaussian mixtures up to a TV-distance of $\eps$ is $\tilde\Theta(d^2k/\eps^2)$ in general, and $\tilde\Theta(dk/\eps^2)$ if all components are axis-aligned. Unfortunately, their learning algorithms are not computationally efficient and run in time exponential in $k$. For the spherical case, \cite{DK20} gave a more efficient algorithm that runs in quasi-polynomial time $\poly(d)\cdot(k/\eps)^{O(\log^2k)}$. In the proper learning setting (i.e., $\hat P$ must be a Gaussian mixture), \cite{SOAJ14} gave an algorithm that takes $\poly(dk/\eps)$ samples and runs in time $\poly(d)\cdot(k/\eps)^{O(k^2)}$.

\paragraph{Super-resolution, and mixture learning using Fourier analysis.} Finally, we acknowledge that the Fourier approach that we explore in this work is fairly natural, and similar ideas have appeared in prior work. In particular, our approach is closely related to the super-resolution problem (\cite{Don92,CF13,CF14})---to recover $k$ unknown locations $\mu_1, \ldots, \mu_k \in \R^d$ from (exact or noisy) observations of the form $\sum_{j=1}^ke^{i\mu_j\top\xi}$ for $\|\xi\| \le M$, where $M$ is called the cutoff frequency, and the norm is typically $\ell_2$ or $\ell_\infty$. In comparison, our approach (outlined in Section~\ref{sec:proof-overview}) differs in two aspects: (1) We focus on a simpler testing version of the problem---to decide whether one of the locations is near a given reference point; (2) The Gaussian truncation significantly down-weights the observation coming from the points that are far from the reference point.

For the $d=1$ case, \cite{Moi15} showed that a cutoff frequency of $M = O(1/\Delta)$ suffices if the points are $\Delta$-separated. This implies an algorithm that efficiently learns spherical Gaussian mixtures in one dimension under separation $1/\sqrt{\log k}$, which is (nearly) recovered by our positive result. For the general case, \cite{HK15} gave an algorithm that provably works for ($\ell_2$) cutoff frequency $M = O((\sqrt{d\log k} + \log k)/\Delta)$. When applied to learning GMMs, however, their algorithm requires separation $\sqrt{d}$. It is conceivable that the algorithm of~\cite{HK15}, equipped with appropriate modifications and a tighter analysis (e.g., the approach of~\cite{CM21}), \emph{might} give a guarantee similar to ours, but no such analysis is explicit in the literature to the best of our knowledge. Moreover, our approach leads to an arguably simpler algorithm with an elementary analysis.

More recently, \cite{CN20} gave an algorithm that is similar to ours for learning mixtures of spherical Gaussians via deconvolving the mixture. However, their algorithm is only shown to work when $\Delta = \Omega(\sqrt{d})$. \cite{CLS20} studied the problem of learning mixtures of linear regressions (MLRs), which can be reduced to estimating the minimum variance in a mixture of zero-mean Gaussians. They solved this problem by estimating the {\em Fourier moments} --  the moments of the Fourier transform, and gave the first sub-exponential time algorithm for learning MLRs. \cite{CM21} studied learning mixtures of Airy disks, a problem that is motivated by the physics of diffraction. Their algorithm also proceeds by first estimating the Fourier transform of the mixture, and then dividing it pointwise by the Fourier spectrum of the ``base'' distribution.

\subsection{Organization of the Paper}
In \Cref{sec:learning-to-testing}, we formally state our main algorithmic result as well as our main technical theorem, which addresses a testing version of the problem. We then prove our upper bound via a simple reduction from parameter learning to testing. In \Cref{sec:upper-gaussian-testing}, we give a simple algorithm that solves this testing version by examining the Fourier transform of the mixture.

The pseudocode of our algorithms are presented in \Cref{sec:pseudocode}. We defer a few technical proofs to Appendices \ref{sec:auxiliary}~and~\ref{sec:gaussian-omitted}. We prove our lower bound result (\Cref{thm:lower-gaussian-informal}) in \Cref{sec:lower}. Finally, in \Cref{sec:upper-general}, we formally state and prove the guarantees for learning non-Gaussian mixtures and present a few applications of this result.

\section{From Learning to Testing}\label{sec:learning-to-testing}
We first state our main positive result -- the formal version of \Cref{thm:upper-gaussian-informal}.

\begin{theorem}\label{thm:upper-gaussian}
    Let $P = \frac{1}{k}\sum_{j=1}^{k}\N(\mu_j, I_d)$ be a uniform mixture of $k \ge 2$ spherical Gaussians with $\Delta$-separated means in $\R^d$ and $\eps < \min\{\Delta / 100, \Delta / (32\sqrt{\min\{d, \ln k\}})\}$. There is an algorithm (Algorithm~\ref{algo:gaussian-learn}) that, given samples from $P$, outputs $k$ vectors that are $\eps$-close to $\mu_1, \ldots, \mu_k$ with high probability. The runtime (and thus the sample complexity) of the algorithm is upper bounded by
    \[
        O((\Delta/\eps)^4k^3\log^2k)\cdot \max\{(d/\eps)^{O(d)}, d^{O(d)}\} \cdot e^{O(M^2)},
    \]
    where $M^2 \lesssim \frac{1}{\Delta^2}(\min\{d, \log k\} + \log\frac{\Delta}{\eps}) \cdot (\min\{d + \log k, d\log(2 + \frac{d}{\Delta^2})\} + \log\frac{\Delta}{\eps})$.
\end{theorem}

\begin{remark}\label{remark:upper-gaussian}
    When $d = O\left(\frac{\log k}{\log\log k}\right)$ and $\Delta/\eps = \Theta(\sqrt{d})$, the runtime can be simplified into
    \[
        \poly(k)\cdot\exp\left(O\left(\frac{d}{\Delta^2}\min\left\{\log k, d\log\left(2 + \frac{d}{\Delta^2}\right)\right\}\right)\right),
    \]
    which is $\poly(k)$ if $\Delta = \Omega\left(\frac{d}{\sqrt{\log k}}\cdot \sqrt{\log \frac{\log k}{d}}\right)$. When $d = o(\log k)$, this condition is strictly looser than both $\Delta = \Omega(\sqrt{d})$ and $\Delta = \Omega(\sqrt{\log k})$.
\end{remark}

We prove \Cref{thm:upper-gaussian} by reducing the parameter learning problem into the following testing version: Given samples from mixture $P$, determine whether $P$ contains a cluster with a mean that is close to a given guess $\mu^* \in \R^d$. Formally, we prove the following theorem:

\begin{theorem}\label{thm:upper-gaussian-testing}
    Let $P = \frac{1}{k}\sum_{j=1}^{k}\N(\mu_j, I_d)$ be a uniform mixture of $k \ge 2$ spherical Gaussians with $\Delta$-separated means in $\R^d$. Let $\eps < \min\{\Delta / 100, \Delta / (32\sqrt{\min\{d, \ln k\}})\}$ and $\mu^* \in \R^d$. There is an algorithm (Algorithm~\ref{algo:gaussian-test}) that, given samples from $P$, either ``accepts'' or ``rejects'', such that it:
    \begin{itemize}
        \item Accepts with probability $\ge 2/3$ if $\min_{j \in [k]}\|\mu_j - \mu^*\|_2 \le \eps / 2$.
        \item Rejects with probability $\ge 2/3$ if $\min_{j \in [k]}\|\mu_j - \mu^*\|_2 \ge \eps$.
    \end{itemize}
    The runtime (and thus the sample complexity) of the algorithm is upper bounded by $O(k^2(\Delta/\eps)^4)\cdot e^{O(d + M^2)}$, where $M^2 \lesssim \frac{1}{\Delta^2}(\min\{d, \log k\} + \log\frac{\Delta}{\eps}) \cdot (\min\{d + \log k, d\log(2 + \frac{d}{\Delta^2})\} + \log\frac{\Delta}{\eps})$.
\end{theorem}

%Assuming the above theorem, the proof of our main theorem is straightforward and is thus deferred to Appendix~\ref{sec:learning-to-testing-omitted}. The proof proceeds by first drawing a few samples $X_1, \ldots, X_N$ from the mixture, and then running the tester (from \Cref{thm:upper-gaussian-testing}) to decide whether each $X_i$ is close to one of the mean vectors of the mixture. Then, a simple argument shows that the samples that the tester accepts can be easily clustered to recover the means. Note that if a sample $X_i$ comes from the $j$-th component $\mathcal{N}(\mu_j, I_d)$, there is a decent probability (e.g., $\Omega(\epsilon)$ when $d = 1$) that $X_i$ is $\epsilon$-close to $\mu_j$. Thus, we can guarantee that $X_1, \ldots, X_N$ contain good guesses for all the $k$ means for moderately large $N$.

We first prove our main theorem assuming \Cref{thm:upper-gaussian-testing}, essentially via an exhaustive search over the means. (We prove \Cref{thm:upper-gaussian-testing} in \Cref{sec:upper-gaussian-testing}.)

\begin{proof}[Proof of \Cref{thm:upper-gaussian}] The proof proceeds by first generating many candidates for the means $\mu_1, \ldots, \mu_k$, and then verifying them using the tester from \Cref{thm:upper-gaussian-testing}.

\paragraph{Finding candidates.} We first show that if we draw sufficiently many samples from $P$, for every mean vector $\mu_j$ there is a sample that is $O(\eps)$-close to it. Indeed, for any $j \in [k]$, the probability that a sample from $P$ is $(\eps/2)$-close to $\mu_j$ is at least
    \begin{align*}
        \frac{1}{k}\cdot\pr{X\sim\N(\mu_j, I_d)}{\|X - \mu_j\|_2 \le \eps/2}
    \ge &~\frac{1}{k}\cdot\max_{0 \le r \le \eps/2}\left[\Vol_d(r)\cdot\frac{1}{(2\pi)^{d/2}}e^{-r^2/2}\right]\\
    =   &~\frac{1}{k}\cdot\max_{0 \le r \le \eps/2}\left[\frac{\pi^{d/2}r^d}{\Gamma(d/2+1)}\cdot\frac{e^{-r^2/2}}{(2\pi)^{d/2}}\right] \tag{Lemma~\ref{lemma:ball-volume}}\\
    =   &~\frac{1}{2^{d/2}k\cdot\Gamma(d/2+1)}\cdot\max_{0 \le r \le \eps/2}r^de^{-r^2/2} \eqqcolon p.
    \end{align*}
    Our algorithm first draws $N \coloneqq (2\ln k)/p$ points $X_1, X_2, \ldots, X_N$ from $P$. The probability that none of them is $(\eps/2)$-close to $\mu_j$ is at most $(1 - p)^{N} \le 1/k^2$. By a union bound, w.h.p.\ it holds for every $j\in [k]$ that some $X_i$ is $(\eps/2)$-close to $\mu_j$.
    
    \paragraph{Running the tester.} Then, we run the testing algorithm from \Cref{thm:upper-gaussian-testing} with $\mu^*$ set to each of the $N$ candidates $X_1, X_2, \ldots, X_N$. For each point $X_i$, we repeat the tester $\Theta(\log N)$ times (with new samples each time) and take the majority vote of the decisions. By a Chernoff bound, w.h.p.\ the decisions are simultaneously correct for all the $N$ points, i.e., for each $X_i$ the majority vote ``accepts'' if $\min_{j\in[k]}\|\mu_j - X_i\|_2 \le \eps / 2$ and ``rejects'' if this minimum distance is at least $\eps$.

    \paragraph{Clustering.} We focus on $C \coloneqq \{i \in [N]: \text{majority vote accepts }X_i\}$ and define $C_j \coloneqq \{i \in C: \|X_i - \mu_j\|_2 \le \eps\}$ for $j \in [k]$. The correctness of the tester implies that for every $i \in C$, $X_i$ must be $\eps$-close to some mean vector $\mu_j$, so $\bigcup_{j \in [k]}C_j = C$. Furthermore, we claim that $\{C_j\}_{j \in [k]}$ are pairwise disjoint and thus constitute a partition of $C$. Suppose towards a contradiction that some $i \in C_j \cap C_{j'}$. Then, since $\eps < \Delta/100$, we get $\|\mu_j - \mu_{j'}\|_2 \le \|\mu_j - X_i\|_2 + \|\mu_{j'} - X_i\|_2 \le 2\eps < \Delta$, a contradiction. We also claim that every $C_j$ is non-empty. This is because we proved earlier that at least one of the $X_i$'s is $\eps/2$-close to $\mu_j$, and the majority vote must accept such an $X_i$.
    
    It remains to show that we can easily identify such a partition $\{C_j\}$ without knowing $\mu_1, \ldots, \mu_k$. We note that for any $i \in C_j$ and $i' \in C_{j'}$. If $j = j'$, we have
    \[\|X_i - X_{i'}\|_2 \le \|X_i - \mu_j\|_2 + \|X_{i'} - \mu_j\|_2 \le 2\eps.\]
    If $j \ne j'$, the condition that $\eps < \Delta/100$ gives
    \[\|X_i - X_{i'}\|_2 \ge \|\mu_j - \mu_{j'}\| - \|X_i - \mu_j\|_2 - \|X_{i'} - \mu_{j'}\|_2 \ge \Delta - 2\eps > 2\eps.\]
    Therefore, if we cluster $C$ such that $i, i' \in C$ belong to the same cluster if and only if $\|X_i - X_{i'}\|_2 \le 2\eps$, we obtain the exact partition $\{C_j\}$. We can then recover $\mu_1, \ldots, \mu_k$ up to an error of $\eps$ by outputting an arbitrary element in each of the $k$ clusters.
    
    \paragraph{Runtime.} To upper bound the runtime, we note that
    \begin{align*}
        N
    &=   (2\ln k) / p
    =   O(k\log k) \cdot 2^{d/2}\Gamma(d/2+1)\cdot \min_{0 \le r \le \eps/2}e^{r^2/2}r^{-d}\\
    &=   O(k\log k)\cdot d^{O(d)}\cdot\begin{cases} (2/\eps)^de^{\eps^2/8}, & \eps \le 2\sqrt{d},\\
    (e/d)^{d/2}, & \eps > 2\sqrt{d}
    \end{cases} \tag{Lemma~\ref{lemma:ball-volume}}\\
    &=  O(k \log k)\cdot\max\{(d/\eps)^{O(d)}, d^{O(d)}\}.
    \end{align*}
    Therefore, the runtime of the algorithm is upper bounded by
    \[
        O(N\log N)
    =   O(k\log^2k) \cdot \max\left\{(d/\eps)^{O(d)}, d^{O(d)}\right\}
    \]
    times that of the tester from \Cref{thm:upper-gaussian-testing}. This finishes the proof.
\end{proof}

\section{Solve the Testing Problem using Fourier Transform}\label{sec:upper-gaussian-testing}
Now we solve the testing problem and prove \Cref{thm:upper-gaussian-testing}. We assume without loss of generality that $\mu^* = 0$, since we can reduce to this case by subtracting $\mu^*$ from each sample from $P$. We also re-index the unknown means $\mu_1, \ldots, \mu_k$ so that $\|\mu_j\|_2$ is non-decreasing in $j$. The testing problem is then equivalent to deciding whether $\|\mu_1\|_2 \le \epsilon/2$ or $\|\mu_1\|_2 \ge \eps$.

The following simple lemma, which we prove in \Cref{sec:gaussian-omitted}, gives the Fourier transform of the mixture $P$ when a ``Gaussian truncation'' of $e^{-\|x\|_2^2/2}$ is applied.
\begin{lemma}\label{lemma:Fourier-d-dim}
    For $P = \frac{1}{k}\sum_{j=1}^{k}\N(\mu_j, I_d)$ and any $\xi \in \R^d$,
    \[
        \Ex{X\sim P}{e^{-\|X\|_2^2/2}\cdot e^{i(\xi^{\top}X)}}
    =   \frac{e^{-\|\xi\|_2^2/4}}{2^{d/2}k}\sum_{j=1}^{k}e^{-\|\mu_j\|_2^2/4} \cdot e^{i(\mu_j^{\top}\xi/2)}
    =   \frac{e^{-\|\xi\|_2^2/4}}{2^{d/2}k} \cdot A_{\mu}(\xi),
    \]
    where we define $A_{\mu}(\xi) \coloneqq \sum_{j=1}^{k}e^{-\|\mu_j\|_2^2/4} \cdot e^{i(\mu_j^{\top}\xi/2)}$.
\end{lemma}

We will show that $A_{\mu}(\xi)$ behaves quite differently depending on whether $\|\mu_1\|_2 \le \eps/2$ or $\|\mu_1\|_2 \ge \eps$. Thus, we can solve the testing problem by estimating $A_{\mu}(\xi)$ for carefully chosen $\xi$. Concretely, we will draw $\xi$ from $\N(0, \sigma^2I_d)$ and then truncate it to $B(0, M):= \{x \in \mathbb{R}^d : \|x\|_2 \leq M\}$,
for parameters $M, \sigma > 0$ to be chosen later. Formally, we focus on the following expectation:
\[
    T_{\mu}
\coloneqq
    \Ex{\xi \sim \N(0, \sigma^2 I_d)}{A_{\mu}(\xi)\cdot\1{\|\xi\|_2 \le M}}.
\]
The key step of our proof of \Cref{thm:upper-gaussian-testing} is to show that $T_{\mu}$ is close to $e^{-(\sigma^2/2+1)\|\mu_1\|_2^2/4}$, and is thus helpful for deciding whether $\|\mu_1\|_2 \le \epsilon/2$ or $\|\mu_1\|_2 \ge \eps$. The following lemma, the proof of which is relegated to \Cref{sec:gaussian-omitted}, helps us to bound the difference between $T_{\mu}$ and $e^{-(\sigma^2/2+1)\|\mu_1\|_2^2/4}$.

\begin{lemma}\label{lemma:expected-Amu}
    For any $M, \sigma > 0$ that satisfy $M^2/\sigma^2 \ge 5d$,
    \[
        T_{\mu}
    =   \sum_{j=1}^{k}e^{-(\sigma^2/2+1)\|\mu_j\|_2^2/4} + O\left(e^{-M^2/(5\sigma^2)}\right) \cdot \sum_{j=1}^{k}e^{-\|\mu_j\|_2^2/4},
    \]
    where the $O(x)$ notation hides a complex number with modulus $\le x$.
\end{lemma}

By Lemma~\ref{lemma:expected-Amu}, the difference $T_{\mu} - e^{-(\sigma^2/2+1)\|\mu_1\|_2^2/4}$ is given by
\[
    \sum_{j=2}^{k}e^{-(\sigma^2/2+1)\|\mu_j\|_2^2/4} + O\left(e^{-M^2/(5\sigma^2)}\right)\cdot\sum_{j=1}^{k}e^{-\|\mu_j\|_2^2/4}.
\]
Let $S_1 \coloneqq \sum_{j=2}^{k}e^{-(\sigma^2/2+1)\|\mu_j\|_2^2/4}$ and $S_2 \coloneqq \sum_{j=1}^{k}e^{-\|\mu_j\|_2^2/4}$ denote the two summations above. We have the following upper bounds on $S_1$ and $S_2$, which we prove in \Cref{sec:gaussian-omitted}:
\begin{claim}\label{claim:S1-bound}
    Assuming $(\sigma^2/2 + 1)\Delta^2 \ge 100\min\{\ln k, d\}$,
\[
    S_1 \le 2e^{-(\sigma^2/2 + 1)\Delta^2/64}\cdot \min\{k, 2^d\}.
\]
\end{claim}

\begin{claim}\label{claim:S2-bound}
We have
    \[
        S_2 \le \begin{cases}
        2, & \Delta^2 \ge 100d,\\
        1 + \frac{267d}{\Delta^2} \cdot \max\left\{\left(\frac{32d}{\Delta^2}\right)^{d/2}, 1\right\}, & \Delta^2 < 100d.
        \end{cases}
    \]
Furthermore,
\[    S_2 \le 10\cdot\min\left\{k, 1 + \left(\frac{32d}{\Delta^2}\right)^{d/2+1}\right\}.\]
\end{claim}

Now we put all the pieces together and prove \Cref{thm:upper-gaussian-testing}.

\begin{proof}[Proof of \Cref{thm:upper-gaussian-testing}]~~By Lemma~\ref{lemma:Fourier-d-dim},
    \begin{align*}
        T_{\mu}
    &=  \Ex{\xi \sim \N(0, \sigma^2 I_d)}{A_{\mu}(\xi)\cdot\1{\|\xi\|_2 \le M}}\\
    &=  \Ex{X\sim P\atop \xi \sim \N(0, \sigma^2 I_d)}{2^{d/2}k\cdot e^{\|\xi\|_2^2/4}\cdot e^{-\|X\|_2^2/2}\cdot e^{i(\xi^{\top}X)}\cdot\1{\|\xi\|_2 \le M}}.
    \end{align*}
    Since the term inside the expectation has modulus $\le k\cdot e^{O(d + M^2)}$, a Chernoff bound implies that we can estimate $T_{\mu}$ up to any additive error $\gamma > 0$ using $O((k/\gamma)^2)\cdot e^{O(d + M^2)}$ samples from $P$.
    
    By Lemma~\ref{lemma:expected-Amu} and our definition of $S_1$ and $S_2$, assuming $M^2/\sigma^2 \ge 5d$,
    \[
        \left|T_{\mu} - e^{-(\sigma^2/2+1)\|\mu_1\|_2^2/4}\right|
    \le S_1 + {e^{-M^2/(5\sigma^2)}}S_2.
    \]
    In the rest of the proof, we will pick $\sigma$ and $M$ carefully so that both $S_1$ and ${e^{-M^2/(5\sigma^2)}}S_2$ are upper bounded by $\gamma \coloneqq (\sigma^2/2+1)\eps^2/64$.
    Assuming this, we are done: We can simply estimate $T_{\mu}$ up to an additive error of $\gamma$. Let $\widehat{T_{\mu}}$ denote the estimate. We accept if and only if $\Re\widehat{T_{\mu}} \ge \theta \coloneqq \frac{1}{2}\left[e^{-(\sigma^2/2+1)\eps^2/16} + e^{-(\sigma^2/2+1)\eps^2/4}\right]$. Indeed, suppose that $\|\mu_1\|_2 \le \eps / 2$, we have
    \[
        \Re\widehat{T_{\mu}}
    \ge \Re T_{\mu} - \gamma
    \ge e^{-(\sigma^2/2+1)\|\mu_1\|_2^2/4} - (S_1 + e^{-M^2/(5\sigma^2)}S_2) - \gamma
    \ge e^{-(\sigma^2/2+1)\eps^2/16} - 3\gamma.
    \]
    Similarly, $\Re\widehat{T_{\mu}}
    \le e^{-(\sigma^2/2+1)\eps^2/4} + 3\gamma$ if $\|\mu_1\|_2 \ge \eps$.
    If we set $\sigma$ such that $(\sigma^2/2+1)\eps^2 \le 1$, we have
    $
        e^{-(\sigma^2/2+1)\eps^2/16} - e^{-(\sigma^2/2+1)\eps^2/4}
    \ge \frac{1}{8}(\sigma^2/2+1)\eps^2
    =   8\gamma,
    $
    which implies $\Re\widehat{T_{\mu}} > \theta$ in the former case and $\Re\widehat{T_{\mu}} < \theta$ in the latter case. Therefore, our algorithm decides correctly.
    
    \paragraph{Choice of parameters.} Claim~\ref{claim:S1-bound} implies that if we set $\sigma^2/2 + 1
    = \frac{512}{\Delta^2}\left(\min\{d, \ln k\} + \ln\frac{\Delta}{\eps}\right)$,
    we can ensure $S_1 \le (\sigma^2/2+1)\eps^2/64 = \gamma$. Furthermore, this choice of $\sigma$ and the assumption $\eps < \min\{\Delta / 100, \Delta / (32\sqrt{\min\{d, \ln k\}})\}$ guarantee the condition $(\sigma^2/2+1)\eps^2 \le 1$ that we need.
    
    It remains to pick $M$ such that $e^{-M^2/(5\sigma^2)} \cdot S_2 \le \gamma = (\sigma^2/2 + 1)\eps^2/64$.
    We also need $M^2/\sigma^2 \ge 5d$ to ensure that Lemma~\ref{lemma:expected-Amu} can be applied. It suffices to let $M^2 \ge 5\sigma^2\cdot\left(d + \ln S_2 + \ln\frac{64}{(\sigma^2/2 + 1)\eps^2}\right)$.
    Our choice of $\sigma$ guarantees $\ln\frac{64}{(\sigma^2/2 + 1)\eps^2} \le 2\ln\frac{\Delta}{\eps}$, so it is in turn sufficient to pick $M$ such that
    \[
        M^2
    =   \frac{5120}{\Delta^2}\left(\min\{d, \ln k\} + \ln\frac{\Delta}{\eps}\right)\left(d + 2\ln\frac{\Delta}{\eps} + \ln S_2\right).
    \]
    Applying Claim~\ref{claim:S2-bound} shows that $M$ can be chosen such that
    \[
        M^2 \lesssim
        \frac{1}{\Delta^2}\left(\min\{d, \log k\} + \log\frac{\Delta}{\eps}\right) \cdot \left(\min\left\{d + \log k, d\log\left(2 + \frac{d}{\Delta^2}\right)\right\} + \log\frac{\Delta}{\eps}\right).
    \]
    
    \paragraph{Runtime.} The runtime of our algorithm is dominated by the number of samples drawn from $P$---$O((k/\gamma)^2)\cdot e^{O(d+M^2)}$, where $\gamma = \Theta((\sigma^2+1)\eps^2)$. Plugging our choice of $\sigma$ into $\gamma$ gives $1/\gamma = O((\Delta/\eps)^2)$. The runtime can thus be upper bounded by $O(k^2(\Delta/\eps)^4)\cdot e^{O(d + M^2)}$.
\end{proof}

\section*{Acknowledgments}
We would like to thank the anonymous reviewers of earlier versions of this paper for their suggestions on the presentation and for pointers to the literature.

\bibliographystyle{alpha}
\bibliography{main}

\newcommand{\etalchar}[1]{$^{#1}$}
\begin{thebibliography}{ABDH{\etalchar{+}}20}

\bibitem[ABDH{\etalchar{+}}20]{ABH+20}
Hassan Ashtiani, Shai Ben-David, Nicholas~JA Harvey, Christopher Liaw, Abbas
  Mehrabian, and Yaniv Plan.
\newblock Near-optimal sample complexity bounds for robust learning of gaussian
  mixtures via compression schemes.
\newblock {\em Journal of the ACM (JACM)}, 67(6):1--42, 2020.

\bibitem[ABG{\etalchar{+}}14]{ABG+14}
Joseph Anderson, Mikhail Belkin, Navin Goyal, Luis Rademacher, and James Voss.
\newblock The more, the merrier: the blessing of dimensionality for learning
  large gaussian mixtures.
\newblock In {\em Conference on Learning Theory (COLT)}, pages 1135--1164,
  2014.

\bibitem[AK05]{AK05}
Sanjeev Arora and Ravi Kannan.
\newblock Learning mixtures of separated nonspherical gaussians.
\newblock {\em The Annals of Applied Probability}, 15(1A):69--92, 2005.

\bibitem[AM05]{AM05}
Dimitris Achlioptas and Frank McSherry.
\newblock On spectral learning of mixtures of distributions.
\newblock In {\em Conference on Learning Theory (COLT)}, pages 458--469, 2005.

\bibitem[BS10]{BS10}
Mikhail Belkin and Kaushik Sinha.
\newblock Polynomial learning of distribution families.
\newblock In {\em Foundations of Computer Science (FOCS)}, pages 103--112,
  2010.

\bibitem[BV08]{BV08}
Spencer~Charles Brubaker and Santosh Vempala.
\newblock Isotropic pca and affine-invariant clustering.
\newblock In {\em Foundations of Computer Science (FOCS)}, pages 551--560,
  2008.

\bibitem[CFG13]{CF13}
Emmanuel~J Cand{\`e}s and Carlos Fernandez-Granda.
\newblock Super-resolution from noisy data.
\newblock {\em Journal of Fourier Analysis and Applications}, 19(6):1229--1254,
  2013.

\bibitem[CFG14]{CF14}
Emmanuel~J Cand{\`e}s and Carlos Fernandez-Granda.
\newblock Towards a mathematical theory of super-resolution.
\newblock {\em Communications on pure and applied Mathematics}, 67(6):906--956,
  2014.

\bibitem[CLS20]{CLS20}
Sitan Chen, Jerry Li, and Zhao Song.
\newblock Learning mixtures of linear regressions in subexponential time via
  fourier moments.
\newblock In {\em Symposium on Theory of Computing (STOC)}, pages 587--600,
  2020.

\bibitem[CM21]{CM21}
Sitan Chen and Ankur Moitra.
\newblock Algorithmic foundations for the diffraction limit.
\newblock In {\em Symposium on Theory of Computing (STOC)}, pages 490--503,
  2021.

\bibitem[CN20]{CN20}
Somnath Chakraborty and Hariharan Narayanan.
\newblock Learning mixtures of spherical gaussians via fourier analysis.
\newblock {\em arXiv preprint arXiv:2004.05813}, 2020.

\bibitem[Das99]{Das99}
Sanjoy Dasgupta.
\newblock Learning mixtures of gaussians.
\newblock In {\em Foundations of Computer Science (FOCS)}, pages 634--644,
  1999.

\bibitem[DK20]{DK20}
Ilias Diakonikolas and Daniel~M Kane.
\newblock Small covers for near-zero sets of polynomials and learning latent
  variable models.
\newblock In {\em Foundations of Computer Science (FOCS)}, pages 184--195,
  2020.

\bibitem[DKS17]{DKS17}
Ilias Diakonikolas, Daniel~M Kane, and Alistair Stewart.
\newblock Statistical query lower bounds for robust estimation of
  high-dimensional gaussians and gaussian mixtures.
\newblock In {\em Foundations of Computer Science (FOCS)}, pages 73--84, 2017.

\bibitem[DKS18]{DKS18}
Ilias Diakonikolas, Daniel~M Kane, and Alistair Stewart.
\newblock List-decodable robust mean estimation and learning mixtures of
  spherical gaussians.
\newblock In {\em Symposium on Theory of Computing (STOC)}, pages 1047--1060,
  2018.

\bibitem[Don92]{Don92}
David~L Donoho.
\newblock Superresolution via sparsity constraints.
\newblock {\em SIAM Journal on Mathematical Analysis}, 23(5):1309--1331, 1992.

\bibitem[DS07]{DS07}
Sanjoy Dasgupta and Leonard~J Schulman.
\newblock A probabilistic analysis of em for mixtures of separated, spherical
  gaussians.
\newblock {\em Journal of Machine Learning Research (JMLR)}, 8:203--226, 2007.

\bibitem[HK15]{HK15}
Qingqing Huang and Sham~M Kakade.
\newblock Super-resolution off the grid.
\newblock In {\em Advances in Neural Information Processing Systems (NIPS)},
  pages 2665--2673, 2015.

\bibitem[HL18]{HL18}
Samuel~B Hopkins and Jerry Li.
\newblock Mixture models, robustness, and sum of squares proofs.
\newblock In {\em Symposium on Theory of Computing (STOC)}, pages 1021--1034,
  2018.

\bibitem[HP15]{HP15}
Moritz Hardt and Eric Price.
\newblock Tight bounds for learning a mixture of two gaussians.
\newblock In {\em Symposium on Theory of Computing (STOC)}, pages 753--760,
  2015.

\bibitem[KSS18]{KSS18}
Pravesh~K Kothari, Jacob Steinhardt, and David Steurer.
\newblock Robust moment estimation and improved clustering via sum of squares.
\newblock In {\em Symposium on Theory of Computing (STOC)}, pages 1035--1046,
  2018.

\bibitem[KSV05]{KSV05}
Ravindran Kannan, Hadi Salmasian, and Santosh Vempala.
\newblock The spectral method for general mixture models.
\newblock In {\em Conference on Learning Theory (COLT)}, pages 444--457, 2005.

\bibitem[LL22]{LL22}
Allen Liu and Jerry Li.
\newblock Clustering mixtures with almost optimal separation in polynomial
  time.
\newblock In {\em Symposium on Theory of Computing (STOC)}, pages 1248--1261,
  2022.

\bibitem[LM00]{LM00}
Beatrice Laurent and Pascal Massart.
\newblock Adaptive estimation of a quadratic functional by model selection.
\newblock {\em Annals of Statistics}, pages 1302--1338, 2000.

\bibitem[MBZ03]{MBZ03}
Ji{\v{r}}{\'\i} Matou{\v{s}}ek, Anders Bj{\"o}rner, and G{\"u}nter~M Ziegler.
\newblock {\em Using the Borsuk-Ulam theorem: lectures on topological methods
  in combinatorics and geometry}.
\newblock Springer, 2003.

\bibitem[Moi15]{Moi15}
Ankur Moitra.
\newblock Super-resolution, extremal functions and the condition number of
  vandermonde matrices.
\newblock In {\em Symposium on Theory of Computing (STOC)}, pages 821--830,
  2015.

\bibitem[MV10]{MV10}
Ankur Moitra and Gregory Valiant.
\newblock Settling the polynomial learnability of mixtures of gaussians.
\newblock In {\em Foundations of Computer Science (FOCS)}, pages 93--102, 2010.

\bibitem[RV17]{RV17}
Oded Regev and Aravindan Vijayaraghavan.
\newblock On learning mixtures of well-separated gaussians.
\newblock In {\em Foundations of Computer Science (FOCS)}, pages 85--96, 2017.

\bibitem[SOAJ14]{SOAJ14}
Ananda~Theertha Suresh, Alon Orlitsky, Jayadev Acharya, and Ashkan Jafarpour.
\newblock Near-optimal-sample estimators for spherical gaussian mixtures.
\newblock In {\em Advances in Neural Information Processing Systems (NIPS)},
  pages 1395--1403, 2014.

\bibitem[SV89]{SV89}
David~J Smith and Mavina~K Vamanamurthy.
\newblock How small is a unit ball?
\newblock {\em Mathematics Magazine}, 62(2):101--107, 1989.

\bibitem[VW04]{VW04}
Santosh Vempala and Grant Wang.
\newblock A spectral algorithm for learning mixture models.
\newblock {\em Journal of Computer and System Sciences}, 68(4):841--860, 2004.

\end{thebibliography}

\newpage
\appendix

\section{Pseudocode of Algorithms}\label{sec:pseudocode}

The pseudocode of our algorithms for learning and testing spherical Gaussian mixtures follow immediately from the proofs of \Cref{thm:upper-gaussian} and \Cref{thm:upper-gaussian-testing}. In the following, $P$ is a uniform mixture of $k$ spherical Gaussians in $\R^d$ with unknown $\Delta$-separated means.

\begin{algorithm2e}[H]\label{algo:gaussian-learn}
    \caption{Learning Gaussian Mixtures via Testing}
    \KwIn{Sample access to mixture $P$. Accuracy parameter $\eps > 0$.}
    \KwOut{Mean estimates $\hat\mu_1, \hat\mu_2, \ldots, \hat\mu_k$.}
    $p \gets \frac{1}{2^{d/2}k\cdot\Gamma(d/2+1)}\cdot \max_{0\le r\le \eps / 2}r^de^{-r^2/2}$\;
    $N \gets (2\ln k)/p$;
    $C \gets \emptyset$\;
    \For{$i = 1, 2, \ldots, N$} {
        Draw $X_i$ from $P$\;
        Run Algorithm~\ref{algo:gaussian-test} with $\mu^* = X_i$ for $\Theta(\log N)$ times\;
        \uIf{Algorithm~\ref{algo:gaussian-test} accepts more than half of the times}{
            $C \gets C \cup \{X_i\}$\;
        }
    }
    Partition $C$ into $C_1, C_2, \ldots, C_k$ such that $x, y \in C$ are in the same cluster if $\|x - y\|_2 \le 2\eps$\;
    Arbitrarily pick $\hat\mu_1 \in C_1, \hat\mu_2 \in C_2, \ldots, \hat\mu_k \in C_k$\;
    \Return $\hat\mu_1, \ldots, \hat\mu_k$\;
\end{algorithm2e}

\begin{algorithm2e}[H]\label{algo:gaussian-test}
    \caption{Testing Gaussian Mixtures using Fourier Transform}
    \KwIn{Sample access to mixture $P$. Parameters $\eps, \sigma, M > 0$ and candidate mean $\mu^* \in \R^d$.}
    \KwOut{``Accept'' or ``Reject''.}
    $\gamma \gets (\sigma^2/2+1)\eps^2/64$\;
    $N \gets \Theta((k/\gamma)^2)\cdot e^{\Theta(d + M^2)}$\;
    $\mathrm{Avg} \gets 0$\;
    \For{$i = 1, 2, \ldots, N$}{
        Draw $X \sim P$ and $\xi \sim \N(0, \sigma^2I_d)$\;
        \uIf{$\|\xi\|_2 \le M$}{
            $\mathrm{Avg} \gets \mathrm{Avg} + \frac{2^{d/2}k}{N}\cdot e^{\|\xi\|_2^2/4}\cdot e^{-\|X - \mu^*\|_2^2/2}\cdot e^{i\xi^{\top}(X - \mu^*)}$\;
        }
    }
    $\theta \gets \frac{1}{2}\left[e^{-(\sigma^2/2+1)\eps^2/16} + e^{-(\sigma^2/2+1)\eps^2/4}\right]$\;
    \Return ``Accept'' if $\Re\mathrm{Avg} \ge \theta$ and ``Reject'' otherwise\;
\end{algorithm2e}

\section{Auxiliary Lemmas}\label{sec:auxiliary}

\begin{lemma}\label{lemma:factorial}
    For integers $n, k \ge 0$, $k! \ge e^{-k}k^k$ and $\binom{n}{k} \le \left(\frac{en}{k}\right)^k$.
\end{lemma}

\begin{proof}
    By the Taylor expansion of $e^x$, $e^k = \sum_{j=0}^{+\infty}\frac{k^j}{j!} \ge \left.\frac{k^j}{j!}\right|_{j=k} = \frac{k^k}{k!}$. Rearranging gives the first inequality. The second inequality then follows from $\binom{n}{k} = \frac{n(n-1)\cdots(n-k+1)}{k!} \le \frac{n^k}{e^{-k}k^k} = \left(\frac{en}{k}\right)^k$.
\end{proof}

We will need the following standard facts about the volume of an $\ell_2$-ball in high dimensions.

\begin{lemma}\label{lemma:ball-volume}
    The volume of a $d$-ball of radius $r$ is
    \[
        \Vol_d(r)
    =   \frac{\pi^{d/2}}{\Gamma(d/2+1)}\cdot r^d
    \le (2r)^d.
    \]
    Furthermore, $\Gamma(d/2 + 1) = d^{O(d)}$.
\end{lemma}

\begin{proof}
    We refer the reader to~\cite{SV89} for a proof of the identity. The inequality $\Vol_d(r) \le (2r)^d$ holds since the $\ell_2$ ball is contained in the $\ell_{\infty}$ ball of radius $r$, which trivially has volume $(2r)^d$. Finally, the bound $\Gamma(d/2+1) = d^{O(d)}$ follows from Stirling's approximation $\Gamma(z) = (1 + O(1/z))\sqrt{2\pi/z}(z/e)^z$.
\end{proof}

We need the following tail bound of $\chi^2$-distributions proved by~\cite[Equation (4.3)]{LM00}, to control the probability that a Gaussian random variable has a large norm.
\begin{lemma}\label{lemma:chi-square}
    Let random variable $X$ be sampled from the $\chi^2$-distribution with $d$ degrees of freedom. Then, for any $t > 0$,
    \[
        \pr{X}{X \ge d + 2\sqrt{dt} + 2t} \le e^{-t}.
    \]
    Furthermore, for any $t \ge 5d$,
    \[
        \pr{X}{X \ge t} \le e^{-t/5}.
    \]
\end{lemma}

The following classic theorem in topology is used in our lower bound proof. See, e.g., \cite{MBZ03} for a proof of the Borsuk--Ulam theorem. We remark that Borsuk--Ulam has been applied in a similar fashion in the literature of mixture learning~\cite{HP15,CLS20}.
\begin{theorem}[The Borsuk--Ulam Theorem]\label{thm:borsuk-ulam}
    Suppose that $n \ge m$ and $f: \S^n \to \R^m$ is continuous. Then, there exists $x \in \S^n$ such that $f(x) = f(-x)$.
\end{theorem}

\section{Deferred Proofs from Section~\ref{sec:upper-gaussian-testing}}\label{sec:gaussian-omitted}
We start with the proof of Lemma~\ref{lemma:Fourier-d-dim}, which we restate below.

\vspace{6pt}
\noindent\textbf{Lemma~\ref{lemma:Fourier-d-dim}}~~{\it
For $P = \frac{1}{k}\sum_{j=1}^{k}\N(\mu_j, I_d)$ and any $\xi \in \R^d$,
    \[
        \Ex{X\sim P}{e^{-\|X\|_2^2/2}\cdot e^{i(\xi^{\top}X)}}
    =   \frac{e^{-\|\xi\|_2^2/4}}{2^{d/2}k}\sum_{j=1}^{k}e^{-\|\mu_j\|_2^2/4} \cdot e^{i (\mu_j^{\top}\xi/2)}
    =   \frac{e^{-\|\xi\|_2^2/4}}{2^{d/2}k} \cdot A_{\mu}(\xi),
    \]
    where we define $A_{\mu}(\xi) \coloneqq \sum_{j=1}^{k}e^{-\|\mu_j\|_2^2/4} \cdot e^{i (\mu_j^{\top}\xi/2)}$.
}\vspace{6pt}

\begin{proof}[Proof of Lemma~\ref{lemma:Fourier-d-dim}]~~For any $v, \xi \in \R^d$, we have
    \begin{align*}
        \Ex{X \sim \N(v, I_d)}{e^{-\|X\|_2^2/2}\cdot e^{i(\xi^{\top}X)}}
    =   &~\Ex{X \sim \N(v, I_d)}{\prod_{j=1}^{d}e^{-X_j^2/2}\cdot e^{i(\xi_jX_j)}}\\
    =   &~\prod_{j=1}^d\Ex{X_j \sim \N(v_j, 1)}{e^{-X_j^2/2}\cdot e^{i(\xi_jX_j)}}. \tag{$X_1, \ldots, X_d$ are independent}
    \end{align*}
    The $j$-th term in the product above is given by
    \[
        \frac{1}{\sqrt{2\pi}}\int_{-\infty}^{+\infty}e^{-(x-v_j)^2/2 - x^2/2}\cdot e^{i\xi_jx}~\mathrm{d}x
    =   \frac{1}{\sqrt{2}}\cdot e^{-\xi_j^2/4} \cdot e^{-v_j^2/4} \cdot e^{iv_j\xi_j/2},
    \]
    so we have
    \begin{align*}
        \Ex{X \sim \N(v, I_d)}{e^{-\|X\|_2^2/2}\cdot e^{i(\xi^{\top}X)}}
    &=  \prod_{j=1}^{d}\left(\frac{1}{\sqrt{2}}\cdot e^{-\xi_j^2/4} \cdot e^{-v_j^2/4} \cdot e^{iv_j\xi_j/2}\right)\\
    &=  \frac{e^{-\|\xi\|_2^2/4}}{2^{d/2}}\cdot e^{-\|v\|_2^2/4} \cdot e^{i(v^{\top}\xi/2)}.
    \end{align*}
    Finally, the lemma follows from the above identity and averaging over $v \in \{\mu_1, \ldots, \mu_k\}$.
\end{proof}

\vspace{6pt}
\noindent\textbf{Lemma~\ref{lemma:expected-Amu}}{\it~~
    For any $M, \sigma > 0$ that satisfy $M^2/\sigma^2 \ge 5d$,
    \[
        T_{\mu}
    =   \sum_{j=1}^{k}e^{-(\sigma^2/2+1)\|\mu_j\|_2^2/4} + O\left(e^{-M^2/(5\sigma^2)}\right) \cdot \sum_{j=1}^{k}e^{-\|\mu_j\|_2^2/4},
    \]
    where the $O(x)$ notation hides a complex number with modulus $\le x$.
}\vspace{6pt}

\begin{proof}[Proof of Lemma~\ref{lemma:expected-Amu}]~~Recall that $A_{\mu}(\xi) = \sum_{j=1}^{k}e^{-\|\mu_j\|_2^2/4} \cdot e^{i (\mu_j^{\top}\xi/2)}$. The contribution of the $j$-th term of $A_{\mu}(\xi)$ to $T_{\mu}$ is
    \begin{align*}
        &~\Ex{\xi \sim \N(0, \sigma^2 I_d)}{e^{-\|\mu_j\|_2^2/4} \cdot e^{i (\mu_j^{\top}\xi/2)}\cdot\1{\|\xi\|_2 \le M}}\\
    =   &~e^{-\|\mu_j\|_2^2/4} \cdot \left[\Ex{\xi \sim \N(0, \sigma^2 I_d)}{e^{i (\mu_j^{\top}\xi/2)}} + O\left(\pr{\xi \sim \N(0, \sigma^2 I_d)}{\|\xi\|_2 > M}\right)\right]\\
    =   &~e^{-\|\mu_j\|_2^2/4} \cdot \left[e^{-\sigma^2\|\mu_j\|_2^2/8} + O\left(e^{-M^2/(5\sigma^2)}\right)\right] \tag{$M^2/\sigma^2 \ge 5d$ and Lemma~\ref{lemma:chi-square}}\\
    =   &~e^{-(\sigma^2/2+1)\|\mu_j\|_2^2/4} + e^{-\|\mu_j\|_2^2/4} \cdot O\left(e^{-M^2/(5\sigma^2)}\right).
    \end{align*}
    The lemma then follows from a summation over $j \in [k]$.
\end{proof}

To prove Claims \ref{claim:S1-bound}~and~\ref{claim:S2-bound}, we need the following lower bound on the norms of the mean vectors.

\begin{lemma}\label{lemma:norm-lb}
    Suppose that $\mu_1, \mu_2, \ldots, \mu_k \in \R^d$ are $\Delta$-separated, and $\|\mu_1\|_2 \le \|\mu_2\|_2 \le \cdots \le \|\mu_k\|_2$. For any $j \ge 2$,
    \[
        \|\mu_j\|_2 \ge \max\left\{\frac{\Delta}{2}, \frac{\Delta j^{1/d}}{4}\right\}.
    \]
\end{lemma}

\begin{proof}[Proof of Lemma~\ref{lemma:norm-lb}]~~Fix $j \ge 2$. The first lower bound follows from
\[
    \Delta
\le \|\mu_1 - \mu_j\|_2
\le \|\mu_1\|_2 + \|\mu_j\|_2
\le 2\|\mu_j\|_2.
\]
Now we turn to the second bound. Since $\mu_1, \ldots, \mu_j$ are $\Delta$-separated, the $j$ balls $B(\mu_1, \Delta/2)$, $B(\mu_2, \Delta/2)$, $\ldots$, $B(\mu_j, \Delta/2)$ are disjoint and all contained in $B(0, \|\mu_j\|_2 + \Delta/2)$. Thus, we have $j \cdot (\Delta/2)^d \le (\|\mu_j\|_2 + \Delta/2)^d$, which implies $\|\mu_j\|_2 \ge \frac{\Delta}{2}\cdot(j^{1/d} - 1)$.

For $j \ge 2^d$, we have $j^{1/d} / 2 \ge 1$ and the lower bound can be relaxed to
\[
    \|\mu_j\|_2 \ge \frac{\Delta}{2}\cdot \left(j^{1/d} - \frac{j^{1/d}}{2}\right)
=   \frac{\Delta j^{1/d}}{4}.
\]
For $2 \le j < 2^d$, we have $j^{1/d} < 2$, so the first lower bound implies $\|\mu_j\|_2 \ge \frac{\Delta}{2} \ge \frac{\Delta j^{1/d}}{4}$.
\end{proof}

With the lower bounds on $\|\mu_j\|_2$,
we are ready to bound $S_1 = \sum_{j=2}^{k}e^{-(\sigma^2/2+1)\|\mu_j\|_2^2/4}$ and $S_2 = \sum_{j=1}^{k}e^{-\|\mu_j\|_2^2/4}$. We first restate and prove Claim~\ref{claim:S1-bound}:

\vspace{6pt}
\noindent\textbf{Claim~\ref{claim:S1-bound}}{\it~~
    Assuming $(\sigma^2/2 + 1)\Delta^2 \ge 100\min\{\ln k, d\}$,
\[
    S_1 \le 2e^{-(\sigma^2/2 + 1)\Delta^2/64}\cdot \min\{k, 2^d\}.
\]
}\vspace{6pt}

\begin{proof}[Proof of Claim~\ref{claim:S1-bound}]~~If $k \le 2^d$, we apply the bound $\|\mu_j\|_2 \ge \Delta/2$ to all $j \ge 2$ and get
$S_1
\le (k - 1) \cdot e^{-(\sigma^2/2+1)\Delta^2/16}$, which is stronger than what we need. Otherwise, we apply $\|\mu_j\|_2 \ge \Delta j^{1/d}/4$ and get
\begin{align*}
    S_1
&\le \sum_{j=2}^{k}\exp\left(-\frac{(\sigma^2/2+1)\Delta^2j^{2/d}}{64}\right)\\
&\le \sum_{j=2}^{k}\exp\left(-\frac{(\sigma^2/2+1)\Delta^2\lfloor j^{1/d}\rfloor^2}{64}\right) \tag{$j^{1/d} \ge \lfloor j^{1/d}\rfloor$}\\
&=  \sum_{t=1}^{+\infty}\exp\left(-\frac{(\sigma^2/2+1)\Delta^2t^2}{64}\right)\cdot\sum_{j=2}^{k}\1{\lfloor j^{1/d}\rfloor = t}\\
&\le \sum_{t=1}^{+\infty}\exp\left(-\frac{(\sigma^2/2+1)\Delta^2t^2}{64}\right)\cdot (t+1)^d. \tag{$\lfloor j^{1/d}\rfloor = t \implies j < (t + 1)^d$}
\end{align*}
In the last summation, the ratio between the $(t+1)$-th term and the $t$-th term is given by
\[
    \exp\left(-\frac{(\sigma^2/2+1)\Delta^2(2t+1)}{64}\right)\cdot\left(1 + \frac{1}{t+1}\right)^d
\le \exp\left(-\frac{3(\sigma^2/2+1)\Delta^2}{64}\right)\cdot\left(\frac{3}{2}\right)^d,
\]
which is smaller than $1/2$ under the assumption that $(\sigma^2/2+1)\Delta^2 \ge 100\min\{\ln k, d\}$. Thus, the summation is at most twice the first term, i.e.,
    $S_1 \le 2\cdot e^{-(\sigma^2/2+1)\Delta^2/64}\cdot 2^d$.
\end{proof}

Now we restate and prove Claim~\ref{claim:S2-bound}:

\vspace{6pt}
\noindent\textbf{Claim~\ref{claim:S2-bound}}{\it~~
We have
    \[
        S_2 \le \begin{cases}
        2, & \Delta^2 \ge 100d,\\
        1 + \frac{267d}{\Delta^2} \cdot \max\left\{\left(\frac{32d}{\Delta^2}\right)^{d/2}, 1\right\}, & \Delta^2 < 100d.
        \end{cases}
    \]
Furthermore, $S_2 \le 10\cdot\min\left\{k, 1 + \left(\frac{32d}{\Delta^2}\right)^{d/2+1}\right\}$.
}\vspace{6pt}

\begin{proof}[Proof of Claim~\ref{claim:S2-bound}]~~Using $\|\mu_j\|_2 \ge \Delta j^{1/d}/4$ for $j \ge 2$, we can upper bound $S_2$ as follows:
    \begin{align*}
        S_2
    =   \sum_{j=1}^{k}e^{-\|\mu_j\|_2^2/4}
    &\le 1 + \sum_{j=2}^{k}\exp\left(-\frac{\Delta^2j^{2/d}}{64}\right)\\
    &\le 1 + \sum_{j=2}^{k}\exp\left(-\frac{\Delta^2\lfloor j^{2/d}\rfloor}{64}\right) \tag{$j^{2/d} \ge \lfloor j^{2/d}\rfloor$}\\
    &=   1 + \sum_{t=1}^{\lfloor k^{2/d}\rfloor}\exp(-\Delta^2 t/64)\cdot\sum_{j=2}^{k}\1{\lfloor j^{2/d}\rfloor = t}\\
    &\le 1 + \sum_{t=1}^{\lfloor k^{2/d}\rfloor}\exp(-\Delta^2 t/64)\cdot (t+1)^{d/2}. \tag{$\lfloor j^{2/d}\rfloor = t \implies j < (t + 1)^{d/2}$}
    \end{align*}
    In the last summation, the ratio between the $(t+1)$-th term and the $t$-th term is given by
    \[
        R_t \coloneqq
    \frac{\exp(-\Delta^2 (t+1)/64)\cdot (t+2)^{d/2}}{\exp(-\Delta^2 t/64)\cdot (t+1)^{d/2}}
    =   e^{-\Delta^2/64}\cdot\left(1 + \frac{1}{t+1}\right)^{d/2}.
    \]
    
    \paragraph{The first case.} When $\Delta^2 \ge 100d$, for every $t \ge 1$ we have
    \[
        R_t \le
        \exp\left(-\frac{100d}{64}\right) \cdot \left(\frac{3}{2}\right)^{d/2}
        \le \left(e^{-25/16}\cdot\sqrt{\frac{3}{2}}\right)^d
        < 1/2.
    \]
    This implies that the summation $\sum_{t=1}^{\lfloor k^{2/d}\rfloor}\exp(-\Delta^2 t/64)\cdot (t+1)^{d/2}$ is dominated by twice its first term, i.e.,
    \[
        S_2 \le 1 + 2\cdot e^{-\Delta^2/64}\cdot 2^{d/2}
    \le 1 + 2\cdot\left(e^{-25/16}\cdot\sqrt{2}\right)^d
    < 2.
    \]
    
    \paragraph{The second case.} By a straightforward calculation,
    \[
        R_t \le e^{-\Delta^2/128}
    \iff 1 + \frac{1}{t+1} \le \exp\left(\frac{\Delta^2}{64d}\right)
    \impliedby
    e^{\frac{1}{t+1}}
    \le \exp\left(\frac{\Delta^2}{64d}\right)
    \iff t + 1 \ge \frac{64d}{\Delta^2}.
    \]
    In other words, the terms in $\sum_{t=1}^{\lfloor k^{2/d}\rfloor}\exp(-\Delta^2 t/64)\cdot (t+1)^{d/2}$ start to decay at a rate of at least $e^{\Delta^2/128}$ after the first $O(d/\Delta^2)$ terms. Therefore, we have
    \[
        S_2
    \le 1 + \left(\frac{64d}{\Delta^2} + \frac{1}{1 - e^{-\Delta^2/128}}\right)\cdot\max_{t \ge 1}\left[\exp(-\Delta^2 t/64)\cdot (t+1)^{d/2}\right].
    \]
    Using the inequality $1 - e^{-x} \ge \frac{e-1}{e}\cdot\min\{x, 1\}$ for $x \ge 0$, we have
    \begin{align*}
        \frac{1}{1 - e^{-\Delta^2/128}}
    &\le \frac{e}{e-1} \cdot \max\left\{\frac{128}{\Delta^2}, 1\right\}\\
    &\le \frac{e}{e-1} \cdot \max\left\{\frac{128d}{\Delta^2}, \frac{100d}{\Delta^2}\right\}
    <   \frac{203d}{\Delta^2}. \tag{$1 \le d$ and $\Delta^2 < 100d$}
    \end{align*}
    On the other hand, elementary calculus shows that the function $t \mapsto e^{-\Delta^2t/64}\cdot(t+1)^{d/2}$ defined over $[0, +\infty)$ is maximized at $t^* = \frac{32d}{\Delta^2} - 1$ if $\frac{32d}{\Delta^2} \ge 1$, and at $t^* = 0$ otherwise. In either case, the maximum value is upper bounded by $\max\{\left(\frac{32d}{\Delta^2}\right)^{d/2}, 1\}$. Therefore, we conclude that in the second case,
    \[
        S_2 \le 1 + \frac{267d}{\Delta^2} \cdot \max\left\{\left(\frac{32d}{\Delta^2}\right)^{d/2}, 1\right\}.
    \]

    \paragraph{The ``furthermore'' part.} If $\frac{d}{\Delta^2} \le \frac{1}{100}$, the first case implies that
    \[
        S_2
    \le 2
    \le 10\cdot\left[1 + \left(\frac{32d}{\Delta^2}\right)^{d/2+1}\right].
    \]
    If $\frac{d}{\Delta^2} \in (\frac{1}{100}, \frac{1}{32}]$, the bound for the second case reduces to
    \[
        S_2
    \le 1 + \frac{267d}{\Delta^2}
    \le 1 + \frac{267}{32}
    <   10 \cdot \left[1 + \left(\frac{32d}{\Delta^2}\right)^{d/2+1}\right].
    \]
    If $\frac{d}{\Delta^2} > \frac{1}{32}$, the bound for the second case implies
    \[
        S_2
    \le 1 + \frac{267}{32}\cdot\left(\frac{32d}{\Delta^2}\right)^{d/2+1}
    <   10\cdot \left[1 + \left(\frac{32d}{\Delta^2}\right)^{d/2+1}\right].
    \]
    Finally, the ``furthermore'' part follows from the above and the observation that $k$ is a trivial upper bound on $S_2$.
\end{proof}

\section{Proof of Lower Bound}\label{sec:lower}
We first state the formal version of \Cref{thm:lower-gaussian-informal}:

\begin{theorem}\label{thm:lower}
    Suppose that $k \ge 3$, $d \le \frac{\ln k}{\ln\ln k}$, $C \ge 100$, and $\ln(8eC) \le (1 - 1/e)\frac{\ln k}{d}$. Then, there are two mixtures $\tilde P$ and $\tilde Q$ of $k$ spherical Gaussians in $\R^d$, such that for some $\Delta = \Theta\left(\frac{d}{\sqrt{C\log k}}\right)$:
    \begin{itemize}
        \item Both $\tilde P$ and $\tilde Q$ are $\Delta$-separated.
        \item The means of $\tilde P$ and $\tilde Q$ are not $(\Delta/2)$-close.
        \item The total variation distance between $\tilde P$ and $\tilde Q$ satisfies $\dTV(\tilde P, \tilde Q) \le 2k^{-C}$.
    \end{itemize}
\end{theorem}

Setting $\omega(1) = C = k^{o(1/d)}$ in Theorem~\ref{thm:lower} shows that with a separation of $\Delta = o(d/\sqrt{\log k})$, $k^{\omega(1)}$ samples are needed to recover the means up to an $O(\Delta)$ error.

The proof of Theorem~\ref{thm:lower} applies standard moment matching techniques in the literature~\cite{HP15,RV17}. The proof proceeds by first constructing two sets of $\le k$ points in $\R^d$ that have the same lower order moments, and then showing that these matching moments imply that their convolutions with a standard Gaussian are close in TV-distance.

The following lemma states that we can choose the centers of two mixtures $\tilde P$ and $\tilde Q$ such that their low-degree mean moments are identical, whereas their parameters are $\Delta$-apart from each other.

\begin{lemma}\label{lemma:matching-moments}
    Suppose that $R > \Delta > 0$ and $\left[\frac{e(t+d)}{d}\right]^d \le N \le \left(\frac{R}{3\Delta}\right)^d$. There exist $2N$ points $\muP_1, \ldots, \muP_N$, $\muQ_1, \ldots, \muQ_N$ in $B(0, 2R)$ such that:
    \begin{itemize}
        \item $\left\|\muP_i - \muP_j\right\|_2 \ge \Delta$ and $\left\|\muQ_i - \muQ_j\right\|_2 \ge \Delta$ for $i \ne j$.
        \item For any permutation $\pi$ over $[N]$, $\max_{i\in[N]}\left\|\muP_i - \muQ_{\pi(i)}\right\|_2 \ge \Delta$.
        \item For any $t' \in [t]$, $\frac{1}{N}\sum_{i=1}^{N}\left[\muP_i\right]^{\otimes t'} = \frac{1}{N}\sum_{i=1}^{N}\left[\muQ_i\right]^{\otimes t'}$.
    \end{itemize}
\end{lemma}

\begin{proof}[Proof of Lemma~\ref{lemma:matching-moments}]~~Let $S$ be an arbitrary maximal $(3\Delta)$-separated subset of $B(0, R)$. Then, the collection $\{B(\mu, 3\Delta): \mu \in S\}$ must cover the ball $B(0, R)$; otherwise, we could add to $S$ the point that is not covered. This implies $|S| \ge \frac{\Vol_d(R)}{\Vol_d(3\Delta)} = (\frac{R}{3\Delta})^d \ge N$, so we can choose $N$ arbitrary points $\mu_1, \mu_2, \ldots, \mu_N$ from $S$.
    
In the following, we construct the point sets $\{\muP_i\}$ and $\{\muQ_i\}$ by slightly perturbing $\{\mu_1, \ldots, \mu_N\}$. Let $M = \binom{t + d}{d}$. For any $t' \ge 1$, the degree-$t'$ moment tensor in $d$ dimensions has exactly $\binom{d + t' - 1}{d - 1}$ distinct entries. So, among the first $t$ moment tensors, the total number of distinct entries is
\[
    \binom{d}{d - 1} + \binom{d + 1}{d - 1} + \cdots + \binom{d + t - 1}{d - 1}
=   \binom{d + t}{d} - 1 = M - 1.
\]

We define a function $f: \S^{Nd-1} \to \R^{M-1}$ as follows. Given $x \in \S^{Nd-1}$, we group the $Nd$ coordinates of $x$ into $N$ groups, and view them as $N$ points $x_1, x_2, \ldots, x_N \in \R^d$. Let $c(x) \coloneqq \frac{\Delta}{\max_{i\in[N]}\|x_i\|_2}$ and $\mux_i \coloneqq \mu_i + c(x)x_i$. ($c(x)$ is well-defined, since $x \in \S^{Nd - 1}$ guarantees that some $x_i$ is non-zero.) Finally, $f(x) \in \R^{M - 1}$ is defined as the concatenation of the $M - 1$ entries in the first $t$ moment tensors of the uniform distribution over $\{\mux_i\}_{i \in [N]}$.

It can be easily verified that $f$ is continuous. Furthermore, our assumption that $\left[\frac{e(t+d)}{d}\right]^d \le N$ together with Lemma~\ref{lemma:factorial} implies
\[
    Nd - 1
\ge N - 1 \ge \left[\frac{e(t+d)}{d}\right]^d - 1 \ge \binom{t + d}{d} - 1 = M - 1.
\]
Thus, by the Borsuk--Ulam theorem (Theorem~\ref{thm:borsuk-ulam}), there exists $x \in \S^{Nd-1}$ such that $f(x) = f(-x)$.

We prove the lemma by setting $\muP_i = \mux_i$ and $\muQ_i = \mu^{(-x)}_i$. By definition of $c(x)$, we have $\|\muP_i - \mu_i\|_2 = \|\muQ_i - \mu_i\|_2 = c(x)\|x_i\|_2 \le \Delta$ for every $i \in [N]$. Thus, $\|\muP_i\|_2 \le \|\mu_i\|_2 + \|\muP_i - \mu_i\|_2 \le R + \Delta \le 2R$ and similarly $\muQ_i \in B(0, 2R)$ for every $i \in [N]$. Furthermore, for any $i \ne j$,
\[
    \|\muP_i - \muP_j\|_2
\ge \|\mu_i - \mu_j\|_2 - \|\muP_i - \mu_i\|_2 - \|\muP_j - \mu_j\|_2
\ge 3\Delta - \Delta - \Delta
=   \Delta,
\]
and similarly $\|\muQ_i - \muQ_j\|_2 \ge \Delta$. This proves the first condition.

Moreover, we note that for $i^* \in \argmax_{i\in[N]}\|x_i\|_2$, it holds that $c(x)\|x_{i^*}\|_2 = \Delta$. Thus, $\|\muP_{i^*} - \muQ_{i^*}\|_2 = 2c(x)\|x_{i^*}\|_2 = 2\Delta$. On the other hand, for any $j \ne i^*$,
\[
    \|\muP_{i^*} - \muQ_{j}\|_2
\ge \|\mu_{i^*} - \mu_j\|_2 - \|\muP_{i^*} - \mu_{i^*}\|_2 - \|\muQ_j - \mu_j\|_2
\ge 3\Delta - \Delta - \Delta
=   \Delta,
\]
so $\muP_{i^*}$ cannot be matched to any $\muQ_j$ without incurring an error of $\Delta$. This proves the second condition.

Finally, the third condition follows from $f(x) = f(-x)$ and our definition of $f$.
\end{proof}

The following lemma allows us to relate the total variation distance between $\tilde P$ and $\tilde Q$ to their $\ell_2$ distance, which is more Fourier-friendly.

\begin{lemma}\label{lemma:l1-to-l2}
    Let $\tilde P$ and $\tilde Q$ be two mixtures of spherical Gaussians with all means contained in $B(0, 2R)$. Then, for any $\eps \in (0, 1)$ and $R' = 2R + \sqrt{d} + \sqrt{2\ln(1/\eps)}$,
    \[
        \dTV(\tilde P, \tilde Q)
    \le \eps + \frac{\sqrt{\Vol_d(R')}}{2}\|\tilde P - \tilde Q\|_2,
    \]
    where $\Vol_d(r)$ denotes the volume of a $d$-ball with radius $r$.
\end{lemma}

\begin{proof}[Proof of Lemma~\ref{lemma:l1-to-l2}]~~By definition of the TV-distance,
    \begin{equation}\begin{split}\label{eq:dTV-decomp}
        \dTV(\tilde P, \tilde Q)
    &=   \frac{1}{2}\int \left|\tilde P(x) - \tilde Q(x)\right|~\mathrm{d}x\\
    &\le \frac{1}{2}\int_{\|x\|_2 \ge R'}\left[\tilde P(x) + \tilde Q(x)\right]~\mathrm{d}x + \frac{1}{2}\int_{\|x\|_2 \le R'}\left|\tilde P(x) - \tilde Q(x)\right|~\mathrm{d}x.
    \end{split}\end{equation}
    
    To bound the first term above, we note that $\int_{\|x\|_2 \ge R'}\tilde P(x)~\mathrm{d}x$ is exactly $\tilde P(\R^d\setminus B(0, R'))$, the probability that a sample from $\tilde P$ has norm greater than $R'$. Since the mean of every cluster of $\tilde P$ is contained in $B(0, 2R)$, this probability is upper bounded by the probability that a standard Gaussian random variable has norm $\ge R' - 2R$:
    \begin{align*}
        \tilde P\left(\R^d \setminus B(0, R')\right)
    &\le \pr{X \sim N(0, I_d)}{\|X\|_2 \ge R' - 2R}\\
    &=   \pr{X \sim \chi^2(d)}{X \ge (\sqrt{d} + \sqrt{2\ln(1/\eps)})^2}\\
    &\le \pr{X \sim \chi^2(d)}{X \ge d + 2\sqrt{d\ln(1/\eps)} + 2\ln(1/\eps)}\\
    &\le \eps. \tag{Lemma~\ref{lemma:chi-square}}
    \end{align*}
    Similarly, we have $\tilde Q\left(\R^d \setminus B(0, R')\right) \le \eps$, so the first term of Equation~\eqref{eq:dTV-decomp} is upper bounded by $\eps$.
    
    The second term of Equation~\eqref{eq:dTV-decomp} can be bounded in terms of the $\ell_2$ distance between $\tilde P$ and $\tilde Q$ using Cauchy-Schwarz:
    \begin{align*}
        \int_{\|x\|_2 \le R'}\left|\tilde P(x) - \tilde Q(x)\right|~\mathrm{d}x
    &\le \sqrt{\int_{\|x\|_2 \le R'}\left[\tilde P(x) - \tilde Q(x)\right]^2~\mathrm{d}x} \cdot \sqrt{\int_{\|x\|_2 \le R'}1~\mathrm{d}x}\\
    &\le \|\tilde P - \tilde Q\|_2\cdot\sqrt{\Vol_d(R')}.
    \end{align*}
    Plugging the above into Equation~\eqref{eq:dTV-decomp} proves $\dTV(\tilde P, \tilde Q) \le \eps + \frac{\sqrt{\Vol_d(R')}}{2}\|\tilde P - \tilde Q\|_2$.
\end{proof}

The following lemma upper bounds the $\ell_2$-distance between $\tilde P$ and $\tilde Q$ under the assumption that their low-degree moments are equal.

\begin{lemma}\label{lemma:l2-upper-bound}
    Suppose that $t/(4R) \ge \sqrt{5d}$, the supports of $P$ and $Q$ are contained in $B(0, 2R)$, and the first $t$ moment tensors of $P$ and $Q$ are equal. Let $\tilde P = P \ast \N(0, I_d)$ and $\tilde Q = Q \ast \N(0, I_d)$. We have
    \[
        \|\tilde P -\tilde Q\|_2^2
    \le 4\exp\left(-\frac{t^2}{80R^2}\right)
    +   2\left(\frac{t}{4R}\right)^d\cdot\frac{(2R)^{2t}}{t!}.
    \]
\end{lemma}

\begin{proof}[Proof of Lemma~\ref{lemma:l2-upper-bound}]~~By the Plancherel theorem,
\[
    \|\tilde P - \tilde Q\|_2^2
=   \frac{1}{(2\pi)^d}\int\left|(\F\tilde P)(\xi) - (\F\tilde Q)(\xi)\right|^2~\mathrm{d}\xi,
\]
where $(\F\tilde P)(\xi) \coloneqq \int\tilde P(x)e^{i\xi^{\top}x}~\mathrm{d}x$ and $(\F\tilde Q)(\xi) \coloneqq \int\tilde Q(x)e^{i\xi^{\top}x}~\mathrm{d}x$.\footnote{Here the $\frac{1}{(2\pi)^d}$ factor comes from that our definition of the Fourier transform differs from the standard one by a factor of $2\pi$ on the exponent.}

Since $\tilde P$ is the convolution of $P$ and $\N(0, I_d)$, we have
\[
(\F \tilde P)(\xi) = (\F \N(0, I_d))(\xi)\cdot(\F P)(\xi) = e^{-\|\xi\|_2^2/2}\Ex{\mu \sim P}{e^{i\xi^{\top}\mu}}.
\]
Similarly, $(\F \tilde Q)(\xi) = e^{-\|\xi\|_2^2/2}\Ex{\mu \sim Q}{e^{i\xi^{\top}\mu}}$ and thus,
\begin{equation}\label{eq:squared-ell-2}
    \|\tilde P - \tilde Q\|_2^2
=   \frac{1}{(2\pi)^d}\int e^{-\|\xi\|_2^2}\left|\Ex{\mu \sim P}{e^{i\xi^{\top}\mu}} - \Ex{\mu \sim Q}{e^{i\xi^{\top}\mu}}\right|^2~\mathrm{d}\xi.
\end{equation}
Since the first $t$ moments tensors of $P$ and $Q$ are equal, for any $\xi \in \R^d$,
\begin{align*}
    &\left|\Ex{\mu \sim P}{e^{i\xi^{\top}\mu}} - \Ex{\mu \sim Q}{e^{i\xi^{\top}\mu}}\right|\\
=   &\left|\sum_{j=0}^{+\infty}\frac{i^j}{j!}\left[\Ex{\mu \sim P}{(\xi^{\top}\mu)^j} - \Ex{\mu \sim Q}{(\xi^{\top}\mu)^j}\right]\right| \tag{Taylor expansion}\\
\le   &\sum_{j=0}^{+\infty}\frac{1}{j!}\left|\Ex{\mu \sim P}{(\xi^{\top}\mu)^j} - \Ex{\mu \sim Q}{(\xi^{\top}\mu)^j}\right| \tag{triangle inequality}\\
\le   &\sum_{j=t+1}^{+\infty}\frac{1}{j!}\left|\Ex{\mu \sim P}{(\xi^{\top}\mu)^j} - \Ex{\mu \sim Q}{(\xi^{\top}\mu)^j}\right| \tag{identical first $t$ moments}\\
\le &\sum_{j=t+1}^{+\infty}\frac{2}{j!}(2R\|\xi\|_2)^j. \tag{Cauchy-Schwarz}
\end{align*}

If $\|\xi\|_2 \le t/(4R)$, the last summation above can be upper bounded by $2\cdot\frac{(2R\|\xi\|_2)^t}{t!}$. Otherwise, we can use the trivial upper bound $\left|\Ex{\mu \sim P}{e^{i\xi^{\top}\mu}} - \Ex{\mu \sim Q}{e^{i\xi^{\top}\mu}}\right| \le 2$. Plugging these back to Equation~\eqref{eq:squared-ell-2} gives
\[
    \|\tilde P - \tilde Q\|_2^2
\le \frac{1}{(2\pi)^d}\int_{\|\xi\|_2 \ge t/(4R)}4e^{-\|\xi\|_2^2}~\mathrm{d}{\xi} + \frac{1}{(2\pi)^d}\int_{\|\xi\|_2 \le t/(4R)}4e^{-\|\xi\|_2^2}\cdot\frac{(2R\|\xi\|_2)^{2t}}{(t!)^2}~\mathrm{d}{\xi}.
\]
Using Lemma~\ref{lemma:chi-square} and the assumption that $t/(4R) \ge \sqrt{5d}$, the first term above is upper bounded by
\[
    4\cdot\frac{1}{(2\pi)^{d/2}}\int_{\|x\|_2\ge t/(4R)}e^{-\|x\|_2^2/2}~\mathrm{d}x
=   4\pr{X\sim\N(0, I_d)}{\|X\|_2 \ge t/(4R)}
\le 4\exp\left(-\frac{t^2}{80R^2}\right).
\]
To bound the second term, we note that the function $x \mapsto e^{-x}x^t$ is maximized at $x = t$. Thus, the second term is upper bounded by
\[
    \frac{1}{(2\pi)^d}\cdot\Vol_d(t/(4R))\cdot 4e^{-t}\frac{(2R)^{2t}t^t}{(t!)^2}
\le 2\left(\frac{t}{4R}\right)^d\cdot\frac{(2R)^{2t}}{t!}.
\]
Here we use $2^{-d}\Vol_d(r) \le r^d$ (Lemma~\ref{lemma:ball-volume}), $4/\pi^d < 2$, and $e^{-t}t^t \le t!$ (Lemma~\ref{lemma:factorial}).
\end{proof}

Now we are ready to prove Theorem~\ref{thm:lower} by combining Lemmas \ref{lemma:matching-moments}~through~\ref{lemma:l2-upper-bound} with carefully chosen parameters.

\begin{proof}[Proof of Theorem~\ref{thm:lower}]~~Let $R > \Delta > 0$ and integers $N, t$ be parameters to be chosen later. Assuming that $\left[\frac{e(t+d)}{d}\right]^d \le N \le \left(\frac{R}{3\Delta}\right)^d$, by Lemma~\ref{lemma:matching-moments}, there exist two distributions $P$ and $Q$ such that: (1) $P$ (resp.\ $Q$) is the uniform distribution over $N$ points $\muP_1, \ldots, \muP_N$ (resp.\ $\muQ_1, \ldots, \muQ_N$) that are $\Delta$-separated and contained in $B(0, 2R)$; (2) $P$ and $Q$ have the same first $t$ moments; (3) the two mixtures $\tilde P =P \ast \N(0, I_d)$ and $\tilde Q = Q \ast \N(0, I_d)$ are $\Delta$-far from each other in their parameters.

Then, applying Lemmas \ref{lemma:l1-to-l2}~and~\ref{lemma:l2-upper-bound} to $\tilde P$ and $\tilde Q$ gives
\[
    \dTV(\tilde P, \tilde Q)
\le \eps + \frac{\sqrt{\Vol_d(R')}}{2}\cdot\sqrt{4\exp\left(-\frac{t^2}{80R^2}\right) + 2\left(\frac{t}{4R}\right)^d\cdot\frac{(2R)^{2t}}{t!}}.
\]
We will set $\eps = k^{-C}$ and ensure the following:
\[
\sqrt{\Vol_d(R')} \le 1/\eps, \quad
\exp\left(-\frac{t^2}{80R^2}\right) \le \eps^4/2, \quad
\left(\frac{t}{4R}\right)^d\cdot\frac{(2R)^{2t}}{t!} \le \eps^4.
\]
These together imply $\dTV(\tilde P, \tilde Q) \le 2\eps = 2k^{-C}$ as desired.

\paragraph{Choice of parameters.} With some calculation, we can show that the above conditions can be satisfied by setting $t = 4C\ln k$ and $R = \frac{\sqrt{C\ln k}}{10}$. Now we set the separation $\Delta$ so that $\left[\frac{e(t+d)}{d}\right]^d \le N \le \left(\frac{R}{3\Delta}\right)^d$ could hold. Since $d \le \ln k \le t$, it is sufficient to guarantee $\frac{2et}{d} \le \frac{R}{3\Delta}$, which holds for $\Delta = \frac{Rd}{6et} = \frac{d}{240e\sqrt{C\ln k}} = \Theta\left(\frac{d}{\sqrt{C\log k}}\right)$.

It remains to verify a few additional assumptions that are required for applying the lemmas. First, our choice of $N$ must be at most $k$. This is equivalent to
    $k \ge \left(\frac{2et}{d}\right)^d = \left(\frac{8eC\ln k}{d}\right)^d$.
Since for any $d > 0$,
\[
    \left(\frac{\ln k}{d}\right)^d
=   \exp\left(d\ln\ln k - d\ln d\right)
\le \left.\exp\left(d\ln\ln k - d\ln d\right)\right|_{d = e^{-1}\ln k}
=   k^{1/e},
\]
it is then sufficient to have $(8eC)^d \le k^{1-1/e}$, but this is guaranteed by the assumption that $\ln(8eC) \le (1-1/e)\frac{\ln k}{d}$. Second, we need to verify that $R > \Delta$. This is equivalent to $24eC\ln k > d$, which clearly holds given $C \ge 100$ and $d \le \frac{\ln k}{\ln\ln k}$. Finally, we need $t/(4R) \ge \sqrt{5d}$ to apply Lemma~\ref{lemma:l2-upper-bound}. This inequality is equivalent to $20C\ln k \ge d$, which also clearly holds.

\paragraph{Detailed calculation.} We first show that $\sqrt{\Vol_d(R')} \le 1/\eps$. Recall that $R' = 2R + \sqrt{d} + \sqrt{2\ln(1/\eps)} = 2R + \sqrt{d} + \sqrt{2C\ln k}$. We have
\begin{align*}
    \ln \sqrt{\Vol_d(R')}
&\le \frac{d}{2}\ln(4R + 2\sqrt{d} + 2\sqrt{2C\ln k}) \tag{Lemma~\ref{lemma:ball-volume}}\\
&\le   \frac{d}{2}\ln(16R\sqrt{2dC\ln k}) \tag{$x,y,z \ge 2 \implies x+y+z\le xyz$}\\
&\le \frac{\ln k}{2\ln\ln k}\cdot\ln\left[\frac{16\sqrt{2}}{10}\sqrt{C\ln k}\cdot\sqrt{\frac{C\ln^2k}{\ln\ln k}}\right] \tag{choice of $R$ and $d \le \frac{\ln k}{\ln\ln k}$}\\
&\le \frac{\ln k}{2\ln\ln k}\cdot\ln(eC) + \frac{\ln k}{2\ln\ln k}\cdot\frac{3}{2}\ln\ln k \tag{$16\sqrt{2}/10 < e$}\\
&\le \frac{C\ln k}{2} + \frac{3C\ln k}{400}
<   C\ln k. \tag{$C \ge 100$}
\end{align*}
This proves $\sqrt{\Vol_d(R')} \le k^{C} = 1/\eps$.

Then, we ensure that $\exp\left(-\frac{t^2}{80R^2}\right) \le \eps^4/2 = k^{-4C}/2$. This is equivalent to
    $R \le \frac{t}{\sqrt{80\ln(2k^{4C})}}$,
which holds given our choice of $R = \frac{t}{40\sqrt{C\ln k}}$.

Finally, we deal with the constraint that
$\left(\frac{t}{4R}\right)^d\cdot\frac{(2R)^{2t}}{t!} \le \eps^4 = k^{-4C}$. After taking a logarithm on both sides, the above inequality reduces to
\[
    d\ln t + 2t\ln(2R) + 4C\ln k
\le d\ln(4R) + \ln(t!).
\]
Since $\ln(t!) \ge t\ln t - t$ (Lemma~\ref{lemma:factorial}) and $t = 4C\ln k$, it is sufficient to have
\[
    2t\ln(2R) + d\ln t + 2t
\le t\ln t.
\]
We will show that in the left-hand side above, the first term is at most $t\ln t - 3t$, while the second terms is bounded by $t$. This would finish the proof.
Indeed, we have
\[
    2t\ln(2R)
=   t\ln(2R)^2
=   t\ln\frac{C\ln k}{25}
\le t\ln\frac{t}{e^3}
=   t\ln t - 3t.
\]
For the second term, we have
\[
    d\ln t
\le     \frac{\ln k}{\ln\ln k}\cdot \ln(4C\ln k)
\le     \ln k\cdot \ln(4C) + \ln k
<       4C\ln k = t,
\]
so $d\ln t \le t$ holds. This completes the proof of the theorem.
\end{proof}

\section{Extension to Non-Gaussian Mixtures}\label{sec:upper-general}
Given a probability distribution $\D$ over $\R^d$, the location family defined by $\D$ is
\[
    \{\D_\mu: \mu\in\R^d\},
\]
where each $\D_{\mu}$ is the distribution of $X + \mu$ when $X$ is drawn from $\D$. In other words, $\D_{\mu}$ is obtained by translating the distribution $\D$ by $\mu$.

We focus on the following testing problem: Given a mixture $P = \frac{1}{k}\sum_{j=1}^{k}\D_{\mu_j}$ of $k$ distributions from a known location family with $\Delta$-separated locations $\mu_1, \ldots, \mu_k$, we need to determine whether the parameter of some component is close to a given $\mu^* \in \R^d$. We prove the following theorem:
\begin{theorem}\label{thm:upper-general-testing}
    Let $P = \frac{1}{k}\sum_{j=1}^{k}\D_{\mu}$ be a uniform mixture of $k \ge 2$ distributions from the location family defined by $\D$ with $\Delta$-separated locations in $\R^d$. Let $\eps < \min\{\Delta / 32, \Delta / (32\sqrt{\min\{d, \ln k\}})\}$ and $\mu^* \in \R^d$. There is an algorithm that, given distribution $\D$ and samples from $P$, either ``accepts'' or ``rejects'', such that it:
    \begin{itemize}
        \item Accepts with probability $\ge 2/3$ if $\min_{j \in [k]}\|\mu_j - \mu^*\|_2 \le \eps / 2$.
        \item Rejects with probability $\ge 2/3$ if $\min_{j \in [k]}\|\mu_j - \mu^*\|_2 \ge \eps$.
    \end{itemize}
    The runtime (and thus the sample complexity) of the algorithm is upper bounded by
    \[
        O\left(\frac{k^2(\Delta/\eps)^4}{\min_{\|\xi\|_2 \le M}\left|\Ex{X \sim \D}{e^{i\xi^{\top}X}}\right|^2}\right),
    \]
    where $M \lesssim \frac{1}{\Delta}(\sqrt{d\log k} + \sqrt{(d + \log k)\log\frac{\Delta}{\eps}} + \log\frac{\Delta}{\eps})$.
\end{theorem}

We remark that Theorem~\ref{thm:upper-general-testing} together with our proof strategy of Theorem~\ref{thm:upper-gaussian} easily gives the following guarantee for learning the parameters $\mu_1, \mu_2, \ldots, \mu_k$.

\begin{corollary}\label{corollary:upper-general}
    Under the setting of Theorem~\ref{thm:upper-general-testing}, there is an algorithm that, given distribution $\D$ and samples from $P$, outputs $\hat\mu_1, \ldots, \hat\mu_k$ that are w.h.p.\ $\eps$-close to the actual parameters $\mu_1, \ldots, \mu_k$. The runtime of the algorithm is upper bounded by
    \[
        O\left(\frac{(\Delta/\eps)^4\cdot(k^3/\delta)\log(k/\delta)}{\min_{\|\xi\|_2 \le M}\left|\Ex{X \sim \D}{e^{i\xi^{\top}X}}\right|^2}\right),
    \]
    where $\delta = \pr{X \sim \D}{\|X\|_2 \le \eps/2}$ and $M \lesssim \frac{1}{\Delta}\left(\sqrt{d\log k} + \sqrt{(d + \log k)\log\frac{\Delta}{\eps}} + \log\frac{\Delta}{\eps}\right)$.
\end{corollary}

We will generalize the techniques that underlie Theorem~\ref{thm:upper-gaussian-testing}, which focuses on the special case where $\D$ is a spherical Gaussian.

\subsection{Proof of Theorem~\ref{thm:upper-general-testing}}
As in the proof of Theorem~\ref{thm:upper-gaussian-testing}, we first examine the Fourier transform of the mixture $P$. For any $\xi \in \R^d$, we have
\[
    \Ex{X \sim P}{e^{i\xi^{\top}X}}
=   \frac{1}{k}\sum_{j=1}^{k}\Ex{X \sim \D}{e^{i\xi^{\top}(X + \mu_j)}}
=   \frac{(\F\D)(\xi)}{k}A_{\mu}(\xi),
\]
where we define $A_{\mu}(\xi) \coloneqq \sum_{j=1}^{k}e^{i\mu_j^{\top}\xi}$ and shorthand $(\F\D)(\xi) = \Ex{X \sim \D}{e^{i\xi^{\top}X}}$.

The above identity allows us to estimate $A_{\mu}(\xi)$ accurately, as long as the magnitude of the Fourier transform is not too small at frequency $\xi$. Note that, here, the definition of $A_{\mu}(\xi)$ is slightly different from that in Lemma~\ref{lemma:Fourier-d-dim}, which has an extra factor of $e^{-\|\mu_j\|_2^2/4}$. This is because we directly look at the Fourier transform of $P$ without the extra ``Gaussian truncation''.

As in the Gaussian case, we focus on the expectation of $A_{\mu}(\xi)$ when $\xi$ is drawn from a truncated Gaussian distribution. We have the following analogue of Lemma~\ref{lemma:expected-Amu}:
\begin{lemma}\label{lemma:expected-Amu-general}
    For any $M, \sigma > 0$ that satisfy $M^2/\sigma^2 \ge 5d$,
    \[
        T_{\mu}
    \coloneqq \Ex{\xi\sim\N(0, \sigma^2I_d)}{A_{\mu}(\xi)\cdot\1{\|\xi\|_2 \le M}}
    =   \sum_{j=1}^{k}e^{-\sigma^2\|\mu_j\|_2^2/2} + k\cdot O\left(e^{-M^2/(5\sigma^2)}\right),
    \]
    where the $O(x)$ notation hides a complex number with modulus $\le x$.
\end{lemma}

Now we are ready to prove Theorem~\ref{thm:upper-general-testing}.

\begin{proof}[Proof of Theorem~\ref{thm:upper-general-testing}]
    The identity $\Ex{X \sim P}{e^{i\xi^{\top}X}} = \frac{(\F\D)(\xi)}{k}A_{\mu}(\xi)$gives
    \begin{align*}
        T_{\mu}
    =   \Ex{\xi\sim\N(0, \sigma^2I_d)}{A_{\mu}(\xi)\cdot\1{\|\xi\|_2 \le M}}
    =   \Ex{\xi \sim \N(0, \sigma^2I_d)\atop X \sim P}{\frac{ke^{i\xi^{\top}X}}{(\F\D)(\xi)} \cdot \1{\|\xi\|_2 \le M}}.
    \end{align*}
    Let $\delta_M \coloneqq \min_{\|\xi\|_2 \le M}|(\F\D)(\xi)|$. The above together with a Chernoff bound shows that we can estimate $T_{\mu}$ up to any error $\gamma > 0$ using $O(k^2\gamma^{-2}\delta_M^{-2})$ samples.
    
    By Lemma~\ref{lemma:expected-Amu-general},
    \[
        \left|T_{\mu} - e^{-\sigma^2\|\mu_1\|_2^2/2}\right|
    \le \sum_{j=2}^{k}e^{-\sigma^2\|\mu_j\|_2^2/2} + ke^{-M^2/(5\sigma^2)},
    \]
    assuming that $M^2/\sigma^2 \ge 5d$. We will pick $\sigma$ and $M$ carefully, so that both terms above are at most $\gamma \coloneqq \sigma^2\eps^2/32$. Let $\widehat{T_{\mu}}$ be an estimate of $T_{\mu}$ with error $\le \gamma$. In the case that $\|\mu_1\|_2 \le \eps/2$, we get
    \[
        \Re\widehat{T_{\mu}}
    \ge \Re T_{\mu} - \gamma
    \ge e^{-\sigma^2\|\mu_1\|_2^2/2} - 3\gamma
    \ge e^{-\sigma^2\eps^2/8} - 3\gamma.
    \]
    Similarly, we have $\Re\widehat{T_{\mu}} \le e^{-\sigma^2\eps^2/2} + 3\gamma$ when $\|\mu_1\|_2 \ge \eps$. If we pick $\sigma$ such that $\sigma^2\eps^2 \le 1$, we have
    \[
        e^{-\sigma^2\eps^2/8} - e^{-\sigma^2\eps^2/2}
    \ge \frac{1}{4}\sigma^2\eps^2
    =   8\gamma.
    \]
    This means that the range of $\Re\widehat{T_{\mu}}$ when $\|\mu_1\|_2 \le \eps/2$ is disjoint from that when $\|\mu_1\|_2 \ge \eps$, which allows the tester to correctly distinguish the two cases.
    
    Finally, it follows from Claim~\ref{claim:S1-bound} and elementary algebra that all the conditions that we need can be satisfied by picking
    \[
        \sigma^2 = \frac{512}{\Delta^2}\left(\min\{d, \ln k\} + \ln\frac{\Delta}{\eps}\right)
        \quad\text{and}\quad
        M^2
    =   10\sigma^2\left(d + \ln k + \ln\frac{\Delta}{\eps}\right).
    \]
    The time complexity of the testing algorithm is then upper bounded by
    \[
        O\left(k^2\gamma^{-2}\delta_M^{-2}\right)
    \lesssim   
    \frac{k^2(\Delta/\eps)^4}{\min_{\|\xi\|_2 \le M}|(\F\D)(\xi)|^2},
    \]
    where
    \begin{align*}
        M
    &=  \frac{\sqrt{5120}}{\Delta}\sqrt{\left(\min\{d, \ln k\} + \ln\frac{\Delta}{\eps}\right)\left(d + \ln k + \ln\frac{\Delta}{\eps}\right)}\\
    &\lesssim \frac{1}{\Delta}\left(\sqrt{d\log k} + \sqrt{(d + \log k)\log\frac{\Delta}{\eps}} + \log\frac{\Delta}{\eps}\right).
    \end{align*}
\end{proof}

\subsection{Examples}\label{sec:examples}
We give a few concrete applications of Theorem~\ref{thm:upper-general-testing} and Corollary~\ref{corollary:upper-general}. For simplicity, we focus on the parameter regime that $d = 1$, $\Delta = O(1)$ and $\Delta / \eps = O(1)$, where the runtime of the learning algorithm (from Corollary~\ref{corollary:upper-general}) can be simplified into
\[
    O\left(\frac{(k^3/\delta)\log(k/\delta)}{\min_{|\xi| \le M}|\Ex{X \sim \D}{e^{i\xi X}}|^2}\right)
\]
for $\delta = \pr{X \sim \D}{|X| \le \eps / 2}$ and some $M = O\left(\frac{\sqrt{\log k}}{\Delta}\right)$. We note that for all of the following applications, it holds that $\delta = \Omega(\eps) = \Omega(\Delta)$. Also recall that the $\Ex{X \sim \D}{e^{i\xi X}}$ term is simply the characteristic function of $\D$ at $\xi$.

\begin{itemize}
    \item \textbf{Unit-Variance Gaussian.} For $\D = \N(0, 1)$, the characteristic function of $\D$ is given by $e^{-\xi^2/2}$. This implies
    \[
        \min_{|\xi| \le M}\left|\Ex{X \sim \D}{e^{i\xi X}}\right|^2
    =   e^{-M^2},
    \]
    and the runtime reduces to $k^{O(1/\Delta^2)}$. Note that Theorem~\ref{thm:upper-gaussian-testing} gives a runtime of $O(k^2)\cdot e^{O(\min\{\log k, \log(1/\Delta)\}/\Delta^2)}$ for the one-dimensional unit-variance Gaussian case, which is strictly better than $k^{O(1/\Delta^2)}$ when $1/\Delta$ is sub-polynomial.
    \item \textbf{Cauchy distribution.} The Cauchy distribution with location parameter $0$ and scale parameter $1$ has probability density function $f(x) = \frac{1}{\pi(1 + x^2)}$ and characteristic function $\xi \mapsto e^{-|\xi|}$. Thus, we can learn a mixture of $k$ Cauchy distributions (with unit scale) in time
    \[
        O((k^3/\Delta)\log(k/\Delta))\cdot e^{O(\sqrt{\log k} / \Delta)}
    =   O(k^3)\cdot e^{O(\sqrt{\log k} / \Delta)},
    \]
    which is polynomial in $k$ for $\Delta = \Omega(1/\sqrt{\log k})$, and almost cubic in $k$ for $\Delta = \Omega(1)$.
    \item \textbf{Logistic distribution.} The logistic distribution with location $0$ and scale $1$ has PDF $f(x) = \frac{e^{-x}}{(1 + e^{-x})^2}$ and CF $\xi \mapsto \frac{\pi\xi}{\sinh \pi\xi}$. Then, using the inequality $\left|\frac{\sinh(x)}{x}\right| \le e^{|x|}$, we obtain a learning algorithm with the same runtime as in the Cauchy case: $O(k^3)\cdot e^{O(\sqrt{\log k}/\Delta)}$.
    
    \item \textbf{Laplace distribution.} The Laplace distribution with fixed scale $1$ and location $0$ has a characteristic function of $\xi \mapsto \frac{1}{1+\xi^2}$. Applying Corollary~\ref{corollary:upper-general} gives a learning algorithm with runtime
    \[
        O\left((k^3/\Delta)(1 + M^2)^2\log(k/\Delta)\right)
    =   \tilde O\left(k^3/\Delta^5\right),
    \]
    which is polynomial in both $k$ and $1/\Delta$.
\end{itemize}

\subsection{Application \#1: Mixture of Exponential Distributions}\label{sec:exponential}
An important family of distributions over $\R$ that is not included in the location family is the family of exponential distributions, $\{\Exp(\lambda): \lambda > 0\}$. Nevertheless, we can learning mixtures of exponential distributions using Corollary~\ref{corollary:upper-general} and a simple reduction: When $X \sim \Exp(\lambda)$, the random variable $Y \coloneqq -\ln X$ has a density of $f(x) = \lambda e^{-x}\cdot e^{-\lambda e^{-x}}$, which is exactly the Gumbel distribution with location $\ln \lambda$ and scale $1$. The problem is then reduced to learning a mixture of unit-scale Gumbel distributions.

The characteristic function of a Gumbel distribution with location $0$ and scale $1$ is given by
\[
    \xi \mapsto \Gamma(1 - i\xi),
\]
where the Gamma function has modulus
\[
    |\Gamma(1 - i\xi)| = \sqrt{\frac{\pi\xi}{\sinh(\pi\xi)}}
    \ge e^{-(\pi/2)|\xi|}.
\]
Applying Corollary~\ref{corollary:upper-general} then shows that, assuming that $\ln\lambda_1, \ldots, \ln\lambda_k$ are $\Delta$-separated, we can recover these $k$ parameters up to error $O(\Delta)$ in time $O(k^3)\cdot e^{O(\sqrt{\log k}/\Delta)}$. Equivalently, we can recover the parameters $\lambda_1, \ldots, \lambda_k$ up to a multiplicative factor of $e^{O(\Delta)}$ and a permutation.

\subsection{Application \#2: Mixture of Linear Regressions in One Dimension}
In one dimension, a mixture of $k$ linear regressions is specified by $k$ weights $w_1, \ldots, w_k \in \R$. A labeled data point $(X, Y)$ from the model is sampled by drawing $X \sim \N(0, 1)$, $j \sim \Unif([k])$ and drawing $Y \sim \N(w_jx, 1)$.

There is a simple reduction from this setting to learning mixtures of unit-variance Gaussians. Note that conditioning on the realization of $X$ and $j$, $Y/X$ is distributed as $\N(w_j, 1/X^2)$. Thus, if $|X| \ge 1$ and we draw $\delta \sim \N(0, 1 - 1/X^2)$, $Y/X + \delta$ follows a Gaussian distribution with mean $w_j$ and variance $1$. Since $|X| \ge 1$ happens with probability $\Omega(1)$, we reduce the problem of learning $w_1, \ldots, w_k$ to learning the means of a mixture of $k$ unit-variance Gaussians, with a constant factor blowup in the time and sample complexities. Therefore, applying \Cref{thm:upper-gaussian-informal} shows that if the weights $w_1, \ldots, w_k$ are $\Delta$-separated for $\Delta = \Omega\left(\sqrt{\frac{\log\log k}{\log k}}\right)$, we can recover the weights up to an $O(\Delta)$ error in $\poly(k)$ time.

\end{document}